\documentclass{article}

\usepackage{microtype}
\usepackage{graphicx}
\usepackage{subcaption}
\usepackage{booktabs} %

\usepackage{hyperref}

\usepackage[preprint]{icml2026}

\usepackage{amsmath}
\usepackage{amssymb}
\usepackage{mathtools}
\usepackage{amsthm}
\usepackage{framed}
\usepackage{booktabs}

\usepackage[capitalize,noabbrev]{cleveref}

\usepackage{thmtools}
\theoremstyle{plain}
\newtheorem{theorem}{Theorem}[section]
\newtheorem{proposition}[theorem]{Proposition}

\newtheorem{corollary}[theorem]{Corollary}
\theoremstyle{definition}
\newtheorem{definition}[theorem]{Definition}

\theoremstyle{remark}

\usepackage[most]{tcolorbox}
\usepackage{multirow}
\usepackage{url}
\usepackage{xfrac}
\usepackage{fix-cm}  %
\usepackage{tabularx}
\usepackage{makecell}
\usepackage{algorithm}
\usepackage{algpseudocode}
\usepackage{enumitem}

\usepackage{amsmath,amsfonts,bm}

\newcommand{\vectorized}{\mathrm{vec}}

\newcommand{\diag}{\mathrm{diag}}

\def\eqref#1{equation~\ref{#1}}

\def\1{\bm{1}}

\def\vtheta{{\bm{\theta}}}

\def\vd{{\bm{d}}}

\def\vg{{\bm{g}}}

\def\vm{{\bm{m}}}

\def\vu{{\bm{u}}}
\def\vv{{\bm{v}}}

\def\vx{{\bm{x}}}

\def\mA{{\bm{A}}}
\def\mB{{\bm{B}}}
\def\mC{{\bm{C}}}
\def\mD{{\bm{D}}}

\def\mG{{\bm{G}}}

\def\mI{{\bm{I}}}

\def\mL{{\bm{L}}}
\def\mM{{\bm{M}}}

\def\mR{{\bm{R}}}

\def\mU{{\bm{U}}}
\def\mV{{\bm{V}}}
\def\mW{{\bm{W}}}
\def\mX{{\bm{X}}}

\def\mZ{{\bm{Z}}}

\def\mSigma{{\bm{\Sigma}}}

\DeclareMathAlphabet{\mathsfit}{\encodingdefault}{\sfdefault}{m}{sl}
\SetMathAlphabet{\mathsfit}{bold}{\encodingdefault}{\sfdefault}{bx}{n}

\def\gB{{\mathcal{B}}}

\def\gD{{\mathcal{D}}}

\def\gG{{\mathcal{G}}}

\def\gL{{\mathcal{L}}}

\def\gP{{\mathcal{P}}}

\def\gZ{{\mathcal{Z}}}

\newcommand{\E}{\mathbb{E}}

\newcommand{\R}{\mathbb{R}}

\DeclareMathOperator{\sign}{sign}
\DeclareMathOperator{\Tr}{Tr}

\usepackage{amssymb}
\makeatletter
\newcommand*{\transpose}{\bgroup\@ifstar{\mathpalette\@transpose{\mkern-3.5mu}\egroup}{\mathpalette\@transpose\relax\egroup}}
\newcommand*{\@transpose}[2]{\setbox0=\hbox{\m@th$#1#2\intercal$}\raise\dp0\box0}
\makeatother

\usepackage{xspace}
\newcommand*{\ie}{i.e.\@\xspace}
\newcommand*{\eg}{e.g.\@\xspace}
\newcommand*{\cf}{c.f.\@\xspace}
\newcommand*{\wrt}{w.r.t.\@\xspace}

\usepackage{pifont}
\newcolumntype{B}{>{\columncolor{MidnightBlue!10}}l}
\newcolumntype{R}{>{\columncolor{Maroon!10}}l}
\newcolumntype{G}{>{\columncolor{Gray!10}}l}

\usepackage{tikz}
\usetikzlibrary{positioning}
\usepackage{tikz-3dplot}
\usetikzlibrary{arrows.meta, positioning, calc, shapes.geometric}

\usepackage[textsize=tiny]{todonotes}

\icmltitlerunning{Clarifying Shampoo: Adapting Spectral Descent to Stochasticity and the Parameter Trajectory}

\begin{document}

\twocolumn[
  \icmltitle{\texorpdfstring{Clarifying Shampoo: \\ Adapting Spectral Descent to Stochasticity and the Parameter Trajectory}{Clarifying Shampoo: Adapting Spectral Descent to Stochasticity and the Parameter Trajectory}}

  \icmlsetsymbol{equal}{*}

  \begin{icmlauthorlist}
    \icmlauthor{Runa Eschenhagen}{cam}
    \icmlauthor{Anna Cai}{meta}
    \icmlauthor{Tsung-Hsien Lee}{}
    \icmlauthor{Hao-Jun Michael Shi}{meta}
  \end{icmlauthorlist}

  \icmlaffiliation{cam}{University of Cambridge}
  \icmlaffiliation{meta}{Meta Platforms}

  \icmlcorrespondingauthor{Runa Eschenhagen}{re393@cam.ac.uk}
  \icmlcorrespondingauthor{Anna Cai}{annacai@meta.com}

  \icmlkeywords{Machine Learning, ICML}

  \vskip 0.3in

]

\printAffiliationsAndNotice{}  %

\begin{abstract}
    Optimizers leveraging the matrix structure in neural networks, such as Shampoo and Muon, are more data-efficient than element-wise algorithms like Adam and Signum.
    While in specific settings, Shampoo and Muon reduce to spectral descent analogous to how Adam and Signum reduce to sign descent, their general relationship and relative data efficiency under controlled settings remain unclear.
    Through extensive experiments on language models, we demonstrate that Shampoo achieves higher token efficiency than Muon, mirroring Adam's advantage over Signum.
    We show that Shampoo's update applied to weight matrices can be decomposed into an \emph{adapted} Muon update.
    Consistent with this, Shampoo's benefits can be exclusively attributed to its application to weight matrices, challenging interpretations agnostic to parameter shapes.
    This admits a new perspective that also avoids shortcomings of related interpretations based on variance adaptation and whitening:
    rather than enforcing semi-orthogonality as in spectral descent, Shampoo's updates are \emph{time-averaged semi-orthogonal in expectation}.
\end{abstract}

\section{Introduction}
\label{sec:background}

\begin{figure*}[t]
\centering

\newlength{\eqboxheight}
\setlength{\eqboxheight}{1.4cm} 

\newlength{\figheight}
\setlength{\figheight}{3.2cm}

\newlength{\totalheight}
\setlength{\totalheight}{\dimexpr \eqboxheight + \figheight \relax}

\begin{minipage}[t]{0.32\textwidth}
\centering
\textbf{Adam}

\begin{minipage}[t][\totalheight][t]{\linewidth}
    \centering
    \begin{minipage}[c][\eqboxheight][c]{\linewidth}
        \centering
        \footnotesize
        \begin{equation*}
             \underbrace{\mathcolor{MidnightBlue}{\left( |\vm_t| \oslash \sqrt{\vv_t} \right)}}_{\text{\tiny\textcolor{MidnightBlue}{Adaptation}}} \odot \underbrace{\mathcolor{Maroon}{\mathrm{sign}\left( \vm_t \right)}}_{\text{\tiny\textcolor{Maroon}{Signum}}}
        \end{equation*}
    \end{minipage}
    \vfill
    \includegraphics[height=\figheight, keepaspectratio]{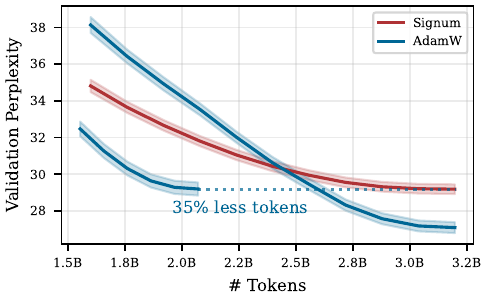}
\end{minipage}
\end{minipage}
\hfill
\begin{minipage}[t]{0.32\textwidth}
\centering
\phantom{\textbf{Adam}} 

\begin{minipage}[t][\totalheight][t]{\linewidth}
    \centering
    
    \vspace*{-2em} 
    
    \resizebox{0.95\linewidth}{!}{%
        \tdplotsetmaincoords{70}{115}
        \begin{tikzpicture}[tdplot_main_coords, transform shape]
        
        \def\x_len{3.5} 
        \def\y_len{3.5} 
        \def\z_len{4.0}
        
        \definecolor{vec_fill}{RGB}{80, 80, 80}  %
        \definecolor{vec_stroke}{RGB}{80, 80, 80}  %
        \definecolor{mat_fill}{RGB}{80, 80, 80}  %
        \definecolor{mat_stroke}{RGB}{80, 80, 80}  %
        \definecolor{axis_col}{RGB}{0, 0, 0}
        \definecolor{beta_col}{RGB}{80, 80, 80}
        \tikzset{
            algo_node/.style={rectangle, rounded corners=2pt, fill=white, draw=gray!30, thin, align=center, font=\sffamily\scriptsize, inner sep=3pt, text=black, fill opacity=0.95},
            axis_arrow/.style={->, >=Latex, ultra thick, color=axis_col},
            dot/.style={circle, inner sep=0pt, minimum size=4pt, fill=white, thick}
        }
    
        \coordinate (O) at (0,0,0); 
        \coordinate (V_BL) at (0,0,0);          \coordinate (V_BR) at (\x_len,0,0);     
        \coordinate (V_TL) at (0,0,\z_len);     \coordinate (V_TR) at (\x_len,0,\z_len);
        \coordinate (M_BL) at (0,\y_len,0);     \coordinate (M_BR) at (\x_len,\y_len,0);     
        \coordinate (M_TL) at (0,\y_len,\z_len); \coordinate (M_TR) at (\x_len,\y_len,\z_len);
    
        \draw[axis_arrow] (O) -- (0,0,\z_len + 1.2) node[anchor=south] {\textbf{First EMA} ($\beta_1$)};
        \draw[axis_arrow] (O) -- (\x_len + 1.2,0,0) node[anchor=north west, xshift=-5pt, yshift=-8pt] {\textbf{Second EMA} ($\beta_2$)};
        \draw[axis_arrow] (O) -- (0,\y_len + 1.5,0) node[anchor=north east, align=right, xshift=-2pt, yshift=-3pt] {\textbf{Geometry}};
    
        \draw[dashed, gray!50, thick] (V_TL) -- (M_TL);
        \draw[dashed, gray!50, thick] (V_TR) -- (M_TR);
        \draw[dashed, gray!50, thick] (V_BR) -- (M_BR);
    
        \fill[vec_fill, opacity=0.25] (V_BL) -- (V_BR) -- (V_TR) -- (V_TL) -- cycle;
        \draw[vec_stroke, thick] (V_BL) -- (V_BR) -- (V_TR) -- (V_TL) -- cycle;
    
        \fill[mat_fill, opacity=0.25] (M_BL) -- (M_BR) -- (M_TR) -- (M_TL) -- cycle;
        \draw[mat_stroke, thick] (M_BL) -- (M_BR) -- (M_TR) -- (M_TL) -- cycle;
    
        \foreach \c in {V_BL, V_BR, V_TL, V_TR} { \node[dot, draw=vec_stroke, fill=vec_stroke] at (\c) {}; }
        \foreach \c in {M_BL, M_BR, M_TL, M_TR} { \node[dot, draw=mat_stroke, fill=mat_stroke] at (\c) {}; }
    
        \node[algo_node, anchor=south west] at (V_BL) {\textbf{SignGD} \\ ($\ell_\infty$ norm) \\[2pt]{\color{beta_col}\tiny $\beta_1=0, \beta_2=0$}};
        \node[algo_node, anchor=south west] at (V_BR) {\textbf{RMSProp}\\[2pt]{\color{beta_col}\tiny $\beta_1=0, \beta_2>0$}};
        \node[algo_node, anchor=south west] at (V_TL) {\textbf{Signum}\\[2pt]{\color{beta_col}\tiny BCOS-m}\\[2pt]{\color{beta_col}\tiny $\beta_1>0, \beta_2=0$}};
        \node[algo_node, anchor=south west] at (V_TR) {\textbf{Adam}\\[2pt]{\color{beta_col}\tiny $\beta_1>0, \beta_2>0$}};
    
        \node[algo_node, anchor=south west] at (M_BL) {\textbf{SpectralGD}\\($S_\infty$ norm)\\[2pt]{\color{beta_col}\tiny $\beta_1=0, \beta_2=0$}};
        \node[algo_node, anchor=south west] at (M_TL) {\textbf{Muon}\\[2pt]{\color{beta_col}\tiny BCOS-m}\\[2pt]{\color{beta_col}\tiny $\beta_1>0, \beta_2=0$}};
        \node[algo_node, anchor=south west] at (M_TR) {\textbf{Shampoo}\\[2pt]{\color{beta_col}\tiny $\beta_1>0, \beta_2>0$}};
    
        \node[font=\bfseries\scriptsize, color=vec_stroke!80!black, rotate=90, anchor=south, yshift=2mm] at ($(V_TL)!0.7!(V_BL)$) {\makecell{Vector \\ Geometry}};
        \node[font=\bfseries\scriptsize, color=mat_stroke!80!black, rotate=90, anchor=south, yshift=2mm] at ($(M_TL)!0.7!(M_BL)$) {\makecell{Matrix \\ Geometry}};
    
        \end{tikzpicture}%
    }
    \vfill
\end{minipage}
\end{minipage}
\hfill
\begin{minipage}[t]{0.32\textwidth}
\centering
\textbf{Shampoo}

\begin{minipage}[t][\totalheight][t]{\linewidth}
    \centering
    \begin{minipage}[c][\eqboxheight][c]{\linewidth}
        \centering
        \footnotesize
        \resizebox{\linewidth}{!}{%
            $ \underbrace{\mathcolor{MidnightBlue}{\mL_t^{-p} \left(\mM_t \mM_t^\transpose \right)^{\frac{1}{4}}}}_{\text{\tiny\textcolor{MidnightBlue}{Left Adaptation}}} \underbrace{\mathcolor{Maroon}{\mU_t \mV_t^\transpose}}_{\text{\tiny\textcolor{Maroon}{Muon}}} \underbrace{\mathcolor{MidnightBlue}{\left(\mM_t^\transpose \mM_t \right)^{\frac{1}{4}} \mR_t^{-p}}}_{\text{\tiny\textcolor{MidnightBlue}{Right Adaptation}}} $%
        }
    \end{minipage}
    \vfill
    \includegraphics[height=\figheight, keepaspectratio]{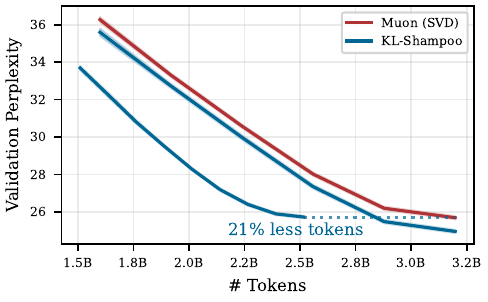}
\end{minipage}
\end{minipage}

\caption{\textbf{Shampoo : Muon :: Adam : Signum.}
    Adam has previously been interpreted as element-wise scaled Signum (\textbf{left, top}), which uses the sign of the gradient's EMA \citep{balles2017dissecting,orvieto2025search}. We show that Shampoo can analogously be understood as Muon, the matrix sign of the gradient's EMA, left- and right-multiplied by matrices (\textbf{right, top}).
    Just like for Adam (\textbf{left, bottom}), the adaptation in Shampoo results in improved token efficiency compared to Muon in a controlled language modeling setting (\textbf{right, bottom}).
    Different variants of the algorithms are analogous in both geometries (\textbf{middle}); see \Cref{tab:optimizer-variants-betas} for the precise relationships.
}
\label{fig:summary}
\end{figure*}
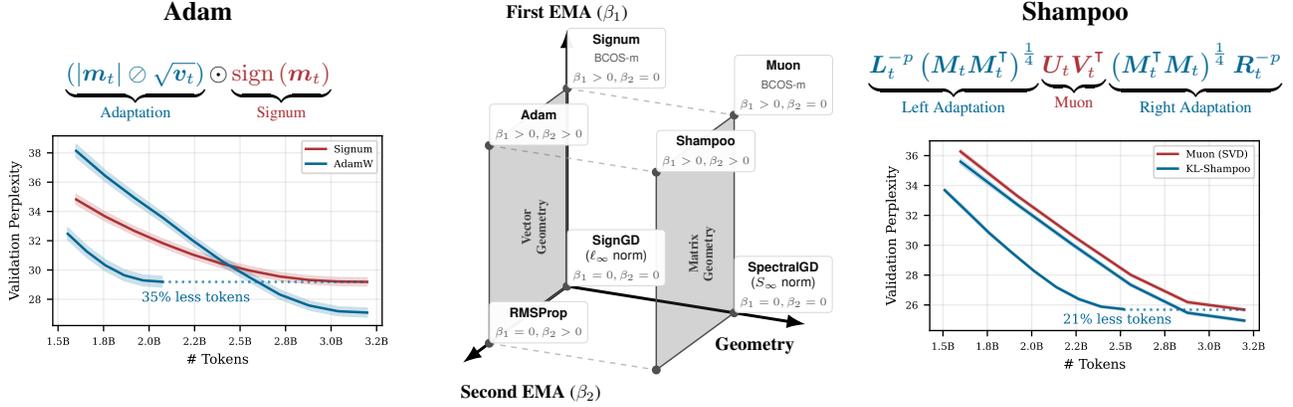

Neural network training is typically modeled as the expected risk minimization problem
\begin{equation}
    \min_{\vtheta \in \R^d} \gL(\vtheta) =  \E_{\gB \sim \gD} \left[ \ell_\gB \left( \vtheta \right) \right],
\end{equation}
where $\gB$ is a mini-batch of data sampled from a stationary distribution $\gD$, and $\ell_\gB: \R^d \rightarrow \R$ is a loss function evaluated on batch $\gB$.
This is optimized using iterative algorithms with the stochastic gradient $\vg_t = \nabla_{\vtheta_t} \ell_{\gB_t}\left( \vtheta_t \right)$ at iteration $t$.\footnote{Because the methods analyzed in this work operate layer-wise, we let $\vtheta_t \in \R^d$ denote the (flattened) parameters of a single layer and $\vg_t \in \R^d$ its corresponding gradient.}

Adam, the de facto standard for neural network training, preconditions the gradient with an element-wise scaling \cite{kingma2014adam}.
Recently, optimizers leveraging the matrix structure in neural networks, like Shampoo and Muon, have shown potential to outperform AdamW \citep{kasimbeg2025accelerating,liu2025muon,chen2025scale}. 
However, these matrix optimizers are not well-understood.

In this work, we show that the relationship between Shampoo and Muon is analogous to that of Adam and Signum; see \Cref{fig:summary}. 
Specifically, we demonstrate that Shampoo is more token efficient than Muon and propose a unified description of how adaptive methods improve upon their non-adaptive counterparts by relaxing their constraints to account for stochasticity and the parameter trajectory.

\subsection{Adam and Sign Descent}

Recall that Adam maintains exponential moving averages (EMAs) of the stochastic gradient and its square \citep{kingma2014adam}:
\begin{equation}
\begin{split}
    \vm_t &= \beta_1 \vm_{t-1} + \left( 1 - \beta_1 \right) \vg_t, \\
    \vv_t &= \beta_2 \vv_{t-1} + \left( 1 - \beta_2 \right) \vg_t^2, \\
    \vtheta_{t+1} &= \vtheta_{t} - \alpha_t \vm_t \oslash (\sqrt{\vv_t} + \epsilon \mathbf{1}),
\end{split}
\end{equation}
where all operations are performed element-wise. 
Here, $\beta_1, \beta_2 \in \left[ 0, 1 \right)$ are the EMA hyperparameters, $\alpha_t > 0$ is the learning rate, and $\epsilon > 0$.\footnote{We ignore bias correction for simplicity, as these terms can be absorbed into the learning rate. The buffers are initialized as $\mathbf{0}$.}

Adam can be interpreted as a variant of sign descent \citep{balles2017dissecting,orvieto2025search}. Specifically, setting $\epsilon = 0$ and assuming all elements of $\vv_t$ are non-zero for all $t$ yields the decomposition:
\begin{equation}
\label{eq:adam-decomposition}
\begin{split}
    \vm_t \oslash\sqrt{\vv_t} &= (|\vm_t| \oslash \sqrt{\vv_t}) \odot \sign\left( \vm_t \right) \\
    &= \underbrace{\frac{1}{\sqrt{1 + \left( \vv_t - \vm_t^2 \right) / \vm_t^2}}}_{\text{Adaptation}} \odot \underbrace{\sign\left( \vm_t \right)}_{\text{Signum}}.
\end{split}
\end{equation}
Ignoring the element-wise adaptation, we recover \emph{Signum}, which updates the weights using the sign of the EMA of the gradients.
Setting $\beta_1 = \beta_2 = \epsilon = 0$ recovers (stochastic) sign descent with $\vg_t / \sqrt{\vg_t^2} = \sign\left( \vg_t \right)$ \citep[SignGD]{bernstein2018signsgd}. Geometrically, this update corresponds to the direction of steepest descent with respect to the $\ell_\infty$ norm in the deterministic setting, but discards the magnitude information that scales the update by $\|\vg_t\|_1$ \citep{kelner2014almostlineartime}.
The connection to sign descent captures aspects of Adam's empirical behavior, such as faster convergence with larger batch sizes \citep{kunstner2023noise} and robustness to heavy-tailed class imbalance \citep{kunstner2024heavytailed}.

\citet{balles2017dissecting} show that, assuming $\vm_t \approx \E[\vg_t]$ and $\vv_t \approx \E[\vg_t^2]$, the element-wise adaptation can be interpreted as a form of \emph{variance adaptation}. Specifically, it contains an estimate of the \emph{relative variance} $(\vv_t - \vm_t^2) / \vm_t^2 \in [0, \infty)$, implying that the scaling of Signum lies in $(0, 1]$.

\subsection{Shampoo}

Shampoo was originally introduced as a Kronecker-factored upper bound for the full-matrix Adagrad preconditioner \citep{gupta2018shampoo}. 
For weight matrix $\mW_t \in \R^{m \times n}$ and corresponding gradient matrix $\mG_t \in \R^{m \times n}$, the simplified Shampoo update with $p, \epsilon > 0$ is\footnote{With $\mathrm{vec}\left( \mW_t \right) = \vtheta_t$ and $\mathrm{vec}\left( \mG_t \right) = \vg_t$. We omit the EMA over the gradient $\mM_t$ and bias correction for simplicity.}
\begin{equation}
\label{eq:shampoo}
\begin{split}
    \mL_t &= \beta_2 \mL_{t-1} + \left( 1 - \beta_2 \right) \mG_t \mG_t^\transpose \; \in \R^{m \times m}, \\
    \mR_t &= \beta_2 \mR_{t-1} + \left( 1 - \beta_2 \right) \mG_t^\transpose \mG_t \; \in \R^{n \times n}, \\
    \mW_{t+1} &= \mW_{t} - \alpha_t (\mL_t + \epsilon \mathbf{I}_m)^{-p} \mG_t (\mR_t + \epsilon \mathbf{I}_n)^{-p}.
\end{split}
\end{equation}
Shampoo was originally introduced with $p = \sfrac{1}{4}$ (Shampoo$^{\sfrac{1}{4}}$) and sum accumulation, but it is often used with $p = \sfrac{1}{2}$ (known as Shampoo$^2$, here Shampoo$^{\sfrac{1}{2}}$) \citep{shi2023distributed,morwani2025new}.
Note that \Cref{eq:shampoo} is equivalent to a vectorized update with a Kronecker-factored preconditioner, $\vtheta_{t + 1} = \vtheta_t - \alpha_t \left( \mR_t\otimes \mL_t \right)^{-p} \vg_t$.

While Shampoo's preconditioner is related to the optimal Kronecker-factored approximation to full-matrix Adam in Frobenius norm \citep{morwani2025new}, \citet{lin2025understanding} recently proposed to minimize the KL divergence of two zero-mean Gaussian distributions,
\begin{equation}
    \label{eq:kl-divergence}
    \min_{\mL \in \R^{m \times m}, \; \mR \in \R^{n \times n}} ~ \mathrm{KL}\left( \E\left[ \vg \vg^\transpose \right] || \mR \otimes \mL \right),
\end{equation}
with optimality condition
\begin{equation}
\begin{aligned}
    \label{eq:kl-optimality}
    \mL &= \E\left[ \mG \mR^{-1} \mG^\transpose \right], & \mR &= \E\left[ \mG^\transpose \mL^{-1} \mG \right].
\end{aligned}
\end{equation}
Solving this problem through gradient descent and using a single sample Monte Carlo approximation to the expectation yields KL-Shampoo, whose factor matrices are defined as%
\begin{equation}
\label{eq:kl-shampoo}
\begin{split}
    \mL_t &= \beta_2 \mL_{t-1} + \left( 1 - \beta_2 \right) \mG_t \mR_{t-1}^{-1} \mG_t^\transpose, \\
    \mR_t &= \beta_2 \mR_{t-1} + \left( 1 - \beta_2 \right) \mG_t^\transpose \mL_{t-1}^{-1} \mG_t. \\
\end{split}
\end{equation}
A similar coupled update was first proposed for natural gradient variational inference \citep{lin2019fast}.
The relations in \Cref{eq:kl-optimality} were also identified as optimality conditions for a certain definition of whitening in \citet{vyas2025improving}.

To reduce runtime, the inverse roots of Shampoo's factor matrices are typically not updated at every iteration. 
To compensate for both the induced staleness and the mis-scaling in the eigenvalues, the per-layer update magnitude is usually grafted from a base optimizer like Adam \citep{agarwal2020disentangling,anil2020scalable,shi2023distributed,eschenhagen2025purifying}.
Alternatively, we can run Adam in Shampoo's eigenbasis \citep[SOAP]{vyas2025soap}, or similarly, decouple and correct Shampoo's eigenvalues (EShampoo) \citep{george2018ekfac,anil2020scalable,eschenhagen2025purifying}.

\subsection{Spectral Descent and Modular Duality}

Analogous to the vector case, where the $\ell_\infty$ norm induces sign descent, one can define steepest descent for matrices using the Schatten-$\infty$ ($S_\infty$) or spectral norm.
Given the reduced SVD of the gradient $\mG_t = \mU_t \mSigma_t \mV_t^\transpose$, we can write (stochastic) spectral descent (SpectralGD) as
\begin{equation}
\label{eq:spectral-descent}
    \mW_{t+1} = \mW_{t} - \alpha_t c(\mG_t) \mU_t \mV_t^\transpose,
\end{equation}
where $c: \R^{m \times n} \rightarrow \R$ is a scaling factor.
The matrix $\mU_t \mV_t^\transpose$ is the unitary factor of the polar decomposition of $\mG_t$; it represents the closest semi-orthogonal matrix to $\mG_t$ in Frobenius norm \citep{schnemann1966generalized}.
It generalizes the sign function to matrices by applying it element-wise to the singular values, \ie, mapping $\bm{\sigma}_t \mapsto \sign(\bm{\sigma}_t)$.
For steepest descent under the spectral norm, we have to scale the update by the nuclear norm of the gradient, \ie, $c(\mG_t) = \| \bm{\sigma}_t \|_1$, where $\bm{\sigma}_t = \diag(\mSigma_t)$, similar to using $\|\vg_t\|_1$ in sign descent.
This method was first proposed for neural network training in \citet{carlson2015stochastic,carlson2015preconditioned}.

Motivated by modular duality \citep{large2024scalable}, \citet{bernstein2025modular} propose using the RMS-RMS operator norm for hidden linear layers, which admits a corresponding spectral descent update with $c(\mG_t) = \sqrt{m / n}$, and recovers the maximal update parameterization \citep[$\mu$P]{yang2021tp5,yang2024spectral}.\footnote{The RMS-RMS operator norm is defined as $\|\mX\|_{\mathrm{RMS}\rightarrow\mathrm{RMS}} = \sqrt{\frac{n}{m}} \|\mX\|_2$, with $||\vx||_\mathrm{RMS} = \frac{1}{\sqrt{n}} ||\vx||_2$.}
\citet{bernstein2024oldoptimizernewnorm} propose to estimate the matrix sign using a Newton-Schulz iteration that only relies on matrix multiplication, making it amenable to parallelization and low-precision data types.
Applying this operation on (Nesterov) momentum yields the popular Muon optimizer \citep{jordan2024muon}.

While not obviously related at first glance, \citet{bernstein2024oldoptimizernewnorm} highlight that Shampoo with $\beta_1 = \beta_2 = \epsilon = 0$ and $p=1/4$ recovers spectral descent:\footnote{Assuming $\mG_t$ has full rank, If $\mG_t$ is rank-deficient, one can recover spectral descent by using the pseudoinverse instead.}
\begin{equation}
\label{eq:shampoo-spectral}
    \mL_t^{-\frac{1}{4}} \mG_t \mR_t^{-\frac{1}{4}} = (\mG_t \mG_t^\transpose)^{-\frac{1}{4}} \mG_t (\mG_t^\transpose \mG_t)^{-\frac{1}{4}} = \mU_t \mV_t^\transpose.
\end{equation}
A similar result holds for one-sided Shampoo with $p = 1/2$.
This is directly analogous to Adam recovering sign descent with the same hyperparameter settings.%

\section{Shampoo : Muon :: Adam : Signum}
\label{sec:analogy}
\begin{table*}[t]
\caption{\textbf{Shampoo : Muon :: Adam : Signum.} Llama 3 architecture trained on C4 data \citep{grattafiori2024llama3,raffel2023exploring}. We report the final validation perplexity (mean $\pm 2\sigma$ across 10 random seeds) for the best hyperparameter setting for different token budgets ($T$), model sizes (measured by the number of parameters $P$), and batch size $B$. The token budgets are defined in \Cref{tab:model-config} in \Cref{sec:experiments-appendix}. For each algorithm, we sweep the learning rate $\alpha_t$ and EMA hyperparameters $\beta_1$ and, where applicable, $\beta_2$. We also sweep $\epsilon$ for Shampoo$^{\sfrac{1}{4}}$, Shampoo$^{\sfrac{1}{2}}$, and KL-Shampoo ($p = \sfrac{1}{2}$, see \Cref{sec:kl-shampoo-implementation} for implementation details). Decoupled weight decay is applied to all parameters and fixed to $0.1$. After 10\% warmup, we use a cosine decay schedule for the learning rate until 0. To ensure a fair comparison, we update Shampoo's preconditioner every iteration using an eigendecomposition, and use SVD for Muon's semi-orthogonalization. All matrix optimizers graft from Adam and are only applied to the hidden weights matrices of the model.}
\label{tab:core-methods}
\centering
\begin{tabular}{lllBBBRBR}
\toprule
$P$ & $T$ & $B$ & Shampoo$^{\sfrac{1}{4}}$ & Shampoo$^{\sfrac{1}{2}}$ & KL-Shampoo & Muon (SVD) & AdamW & Signum \\
\hline
\multirow{3}{*}{320M} & \multirow{2}{*}{$1\times$} & 64 & $25.84 \pm 0.10$ & $25.20 \pm 0.07$ & $25.29 \pm 0.13$ & $25.89 \pm 0.07$ & $27.09 \pm 0.30$ & $29.19 \pm 0.26$ \\
& & 256 & $25.36 \pm 0.11$ & $25.26 \pm 0.05$  & $24.95 \pm 0.09$ & $25.68 \pm 0.09$ & $27.74 \pm 0.56$ & $31.38 \pm 0.47$ \\
& $8\times$ & 256 & $20.62 \pm 0.32$ & $20.52 \pm 0.06$ & $20.33 \pm 0.05$ & $20.67 \pm 0.05$ & $20.93 \pm 0.18$ & $21.50 \pm 0.28$ \\
\hline
1.5B & $1\times$ & 256 & $15.03 \pm 0.04$ & $14.85 \pm 0.02$ & $14.96 \pm 0.04$ & $15.11 \pm 0.02$ & $15.96 \pm 0.12$ & $16.96 \pm 0.18$  \\
\bottomrule
\end{tabular}
\end{table*}

Based on the observation in \Cref{eq:shampoo-spectral}, we can decompose the Shampoo update into a form structurally analogous to Adam's decomposition in \Cref{eq:adam-decomposition}:

\begin{equation}
\label{eq:shampoo-decomposition}
    \mL_t^{-p} \mM_t \mR_t^{-p} = \underbrace{\mL_t^{-p} \left(\mM_t \mM_t^\transpose \right)^{\frac{1}{4}}}_{\text{Left Adaptation}} \underbrace{\mU_t \mV_t^\transpose}_{\text{Muon}} \underbrace{\left(\mM_t^\transpose \mM_t \right)^{\frac{1}{4}} \mR_t^{-p}}_{\text{Right Adaptation}},
\end{equation}
where $\mM_t = \mU_t \mSigma_t \mV_t^\transpose$ is now the reduced SVD of the EMA of the gradient $\mM_t$.
This decomposition reveals that Shampoo consists of two distinct components: (1) the \emph{Muon} update, \ie the \emph{matrix sign} (or polar factor) of $\mM_t$; (2) left and right \emph{adaptation matrices}.
These matrices mirror the element-wise adaptation in Adam, which enhances its token efficiency compared to Signum \citep{orvieto2025search}.
We hypothesize that an analogous choice of the adaptation matrices in Shampoo (determined by $\mL_t$, $\mR_t$, and $p$) should also enhance Muon.

To test this hypothesis, we perform a comprehensive set of experiments on language models comparing Shampoo's and Muon's \emph{token efficiency} by measuring the best achievable validation perplexity given a fixed token budget.
We do not use techniques that improve the memory and computational efficiency of these methods, such as blocking, stale preconditioning, and approximate matrix computations, to isolate the algorithmic contributions of the adaptation matrices.
All optimizers are implemented in the PyTorch Distributed Shampoo codebase \citep{shi2023distributed}, ensuring they all share the same general code path, \eg for momentum and weight decay, differing only in their preconditioners or (matrix) sign operation.
See \Cref{sec:algos-appendix} for pseudocode and implementation details, \Cref{tab:core-methods} and \Cref{sec:language-experiments} for the detailed experimental setup.

\subsection{Results \& Observations}
\label{sec:results}

Our main results are presented in \Cref{tab:core-methods}.

\paragraph{General trends.}
Consistent with \citet{orvieto2025search}, AdamW outperforms Signum across all settings.
As expected, the Shampoo variants and Muon consistently outperform AdamW across all settings \citep{kasimbeg2025accelerating,chen2025scale}.
Confirming our hypothesis, all Shampoo variants match or outperform Muon with SVD. %
Among the Shampoo variants, we observe that KL-Shampoo and Shampoo$^{\sfrac{1}{2}}$ consistently outperform Shampoo$^{\sfrac{1}{4}}$. %

\paragraph{Effect of batch size.}
We observe that Signum and AdamW perform best at a batch size of 64, whereas Shampoo$^{\sfrac{1}{4}}$, KL-Shampoo, and Muon perform better at a batch size of 256.
We tested only two different batch sizes and are likely not at the critical batch size for each method, limiting our ability to draw definitive conclusions.
However, these results suggest, consistent with prior work, that the critical batch size is higher for matrix-based methods than for element-wise methods \citep{vyas2025soap,pethick2025training}.

\paragraph{Effect of model and dataset size.}
Notably, the standard Chinchilla token budget may no longer be compute-optimal for Shampoo and Muon \citep{chen2025scale}.
As the model size or token budget increases, the absolute perplexity gaps between all methods decrease, although the relative ranking of the optimizers remains largely unchanged.
For the 1.5B model, KL-Shampoo is outperformed by Shampoo$^{\sfrac{1}{2}}$, despite being the best-performing method across the other settings with the exception of batch size 64 for the 320M model.
This motivates the hypothesis that KL-Shampoo requires a larger batch size than Shampoo$^{\sfrac{1}{2}}$.

\paragraph{Patterns in optimal hyperparameters.}
Similar to \citet{orvieto2025search}, we observe that the optimal $\beta_1$ and $\beta_2$ hyperparameters in AdamW are correlated, although their values are not necessarily close to each other.
As the batch size increases, the optimal $\beta_1$ and $\beta_2$ hyperparameters decrease, reproducing the observations in \citet{marek2025smallbatch}.
When using the optimal $\beta_2$ from AdamW for grafting, the $\beta_2$ for updating Shampoo's factor matrices, which does not affect the update magnitude, is typically much smaller. %

In contrast to \citet{semenov2025benchmarking}, we generally do not observe that the matrix-based optimizers benefit from a larger learning rate compared to element-wise methods, with the exception of the 320M, 1$\times$ token budget, batch size 256 setting.

\paragraph{Comparison to prior work.}

Most prior benchmarks omit Shampoo, assuming SOAP to be the superior variant.
However, we find that the closely related EShampoo does not generally outperform Shampoo; see \Cref{tab:eshampoo}.
These results align with \citet{vyas2025improving} and \citet{lin2025understanding}, who show that SOAP and Shampoo perform similarly when updating the preconditioner every iteration, and that KL-Shampoo converges faster than KL-SOAP, respectively.\footnote{KL-SOAP runs Adam in KL-Shampoo's eigenbasis.}

Evidence comparing SOAP and Muon is mixed: \citet{wen2025fantastic} prefer Muon for small token budgets and SOAP for larger ones, while \citet{semenov2025benchmarking} find SOAP superior up to 210M parameters, but outperformed by AdamW and Muon at larger scale.
These discrepancies likely stem from the confounders discussed here.
Similarly to us, \citet{frans2025really} observe that SOAP generally outperforms Muon. %

\begin{table}[h]
\caption{EShampoo in the same setting as \Cref{tab:core-methods}; Llama-320M model, $1 \times$ Chinchilla token budget, and 256 batch size.}
\label{tab:eshampoo}
\centering
\begin{tabular}{Gll}
    \toprule
    Shampoo$^{\sfrac{1}{2}}$ & \multicolumn{2}{c}{EShampoo}  \\
    & no grafting & grafting \\
    \hline  
    $25.26 \pm 0.05$ & $25.72 \pm 0.09$ & $26.07 \pm 0.08$ \\
    \bottomrule
\end{tabular}
\end{table}

\subsection{Accounting for Potential Confounders}

\paragraph{Muon's layer-wise scaling.}

Muon is typically used with some layer-wise scaling instead of grafting, \eg the classic scaling $c(\mG_t) = \sqrt{\max{\left(1, m/n\right)}}$ \citep{jordan2024muon} and the Moonlight scaling $c(\mG_t) = 0.2 \cdot \sqrt{\max(m, n)}$ \citep{liu2025muon} that mimics the RMS norm of Adam.
Neither results in a lower validation perplexity than grafting.
\begin{table}[ht]
\caption{Muon (SVD) with classic and Moonlight scaling on the Llama-320M model, $1 \times$ Chinchilla token budget, and 256 batch size, with the same hyperparameter tuning strategy as in \Cref{tab:core-methods}.}
\label{tab:muon-scalings}
\centering
\begin{tabular}{Gll}
    \toprule
    Grafting & Classic & Moonlight  \\
    \hline  
    $25.68 \pm 0.09$ & $25.96 \pm 0.11$ & $25.70 \pm 0.10$ \\
    \bottomrule
\end{tabular}
\end{table}

\paragraph{Matrix computation accuracy.}

Since we use the SVD instead of Newton-Schulz for Muon and tune $\epsilon$ for Shampoo, a natural question is how much this influences our conclusions \citep{nado2021epsilon,amsel2025polarexpress}.
Note that the $\epsilon$ discussed here only affects the matrix root inverse in \Cref{eq:shampoo} and is distinct from the $\epsilon$ used in Adam grafting.
We compare against Shampoo$^{\sfrac{1}{4}}$ and Shampoo$^{\sfrac{1}{2}}$ using the default $\epsilon = 10^{-12}$, and Muon using the Newton-Schulz iteration with the quintic coefficients $(\alpha, \beta, \gamma) = (3.4445, -4.775, 2.0315)$ suggested in \citet{jordan2024muon}.

\begin{table}[ht]
\caption{Shampoo$^{\sfrac{1}{4}}$/Shampoo$^{\sfrac{1}{2}}$ \emph{with the default choice of $\epsilon = 10^{-12}$} and Muon with the standard \emph{Newton-Schulz} (NS) algorithm on the Llama-320M model and $1 \times$ Chinchilla token budget, with the same hyperparameter tuning strategy as in \Cref{tab:core-methods}.}
\label{tab:muon-ns}
\centering
\begin{tabular}{llll}
    \toprule
    $B$ & Shampoo$^{\sfrac{1}{4}}$ & Shampoo$^{\sfrac{1}{2}}$ & Muon (NS) \\
    \hline  
    64 & $26.50 \pm 0.12$ & $25.31 \pm 0.08$ & $26.31 \pm 0.14$ \\
    256 & $26.13 \pm 0.10$ & $25.46 \pm 0.05$ & $26.08 \pm 0.11$ \\
    \bottomrule
\end{tabular}
\end{table}

Tuning $\epsilon$ yields significant gains, shifting the relative ranking of the optimizers.
For example, Shampoo$^{\sfrac{1}{4}}$ performs worse than Muon, whereas their order flips when tuned.
The optimal values of $\epsilon$ differ substantially between the Shampoo variants; see \Cref{sec:epsilon-role}.

\paragraph{Limitations.}
We only consider the C4 dataset and dense Llama 3 architecture with a fixed sequence length, learning rate warmup and schedule, which \citet{semenov2025benchmarking} found to alter their optimizer rankings.
Also, we did not perform an exhaustive grid search over all hyperparameters, \eg we did not tune weight decay; see \Cref{sec:experiments-appendix}.

\begin{tcolorbox}[
    enhanced,
    colback=white,       %
    colframe=black,      %
    coltitle=black,      %
    title=\textbf{Takeaway \#1},
    sharp corners,       %
    boxrule=0.8pt,       %
    bottom=4mm,
    attach boxed title to top left={xshift=0.5cm, yshift*=-0.8\baselineskip},
    boxed title style={
        frame hidden,    %
        colback=white,   %
        left=1mm, right=1mm %
    }
]
    Shampoo applied to weight matrices can be interpreted as a modified Muon algorithm that \emph{achieves superior token efficiency} compared to Muon after controlling for confounders. %
\end{tcolorbox}

\section{Localizing Shampoo's Benefits}
\label{sec:localizing}

To dissect where Shampoo's benefits come from, we:
\begin{enumerate}[label=\arabic*), nosep, topsep=-1.5ex]
    \item isolate Shampoo's effect on non-matrix parameters since \Cref{eq:shampoo-decomposition} only applies to matrices, and
    \item investigate whether both adaptation matrices contribute to its improvements.
\end{enumerate}

\subsection{Revisiting Shampoo Beyond Weight Matrices}
\label{sec:beyond}

Shampoo is commonly thought of as an approximation to (per-parameter) full-matrix Adam, with preconditioner
\begin{equation}
\label{eq:full-matrix-adam}
    \mA_t = \beta_2 \mA_{t-1} + (1-\beta_2) \vg_t \vg_t^\transpose.
\end{equation}
Specifically, Shampoo's preconditioner has been interpreted as an upper bound in Loewner order to full-matrix AdaGrad \citep{gupta2018shampoo}, or an approximation to full-matrix Adam in Frobenius norm \citep{morwani2025new,eschenhagen2025purifying} or KL divergence \citep{lin2025understanding}. 
This perspective suggests Shampoo is shape-agnostic.
However, Shampoo's decomposition in \Cref{eq:shampoo-decomposition} motivates a re-evaluation of this interpretation.

To investigate Shampoo's effectiveness for non-2D parameters, we train a vision transformer and ConvNeXt V2 on Imagewoof using Shampoo with different reshaping strategies (see \Cref{fig:reshape}).
When a parameter is reshaped into a 1D vector, Shampoo mathematically degenerates into per-parameter full-matrix Adam.
Crucially, we find that full-matrix Adam (via 1D reshaping) performs worse than Shampoo (via 2D reshaping) for both architectures.\footnote{We only reshape parameters until the largest dimension is $> 32,768$, so there are still 2D parameters being preconditioned.}
Furthermore, for 4D convolutional kernels, reshaping to 2D matrices outperforms preconditioning higher-order tensors directly (\ie, maintaining the 4D structure).
This implies that Shampoo's benefits are exclusive to the structure of linear layer weight matrices, consistent with recent findings by \citet{eschenhagen2025purifying} and \citet{xie2025structured}.

\begin{figure}[t]
\begin{center}
    \includegraphics[width=\columnwidth]{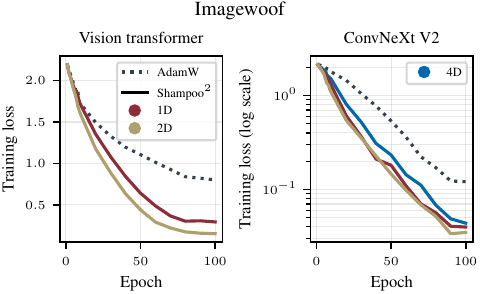}
\end{center}
\caption{Shampoo$^{\sfrac{1}{2}}$ ($p=1/8$ for 4D) with different reshaping strategies. Reshaping to 2D matrices outperforms reshaping to 1D vectors (full-matrix Adam) or preserving 4D tensor structure.}
\label{fig:reshape}
\end{figure}

\begin{table}[ht]
\caption{Shampoo$^{\sfrac{1}{2}}$ applied to different subsets of the parameters on the Llama-320M model, $1 \times$ token budget, and $256$ batch size, and the rest of the setting as in \Cref{tab:core-methods}.}
\label{tab:shampoo-otherparams}
\centering
\begin{tabular}{Gll}
    \toprule
    Shampoo$^{\sfrac{1}{2}}$ & Shampoo$^{\sfrac{1}{2}}$ (1D) & Shampoo$^{\sfrac{1}{2}}$ (embed.) \\
    \hline  
    $25.26 \pm 0.05$ & $25.64 \pm 0.05$ & $25.46 \pm 0.03$ \\
    \bottomrule
\end{tabular}
\end{table}

We further validate this in our language model setting (\Cref{tab:shampoo-otherparams}). 
Additionally applying Shampoo to all 1D parameters or the embeddings and output layers results in worse validation perplexity than applying it exclusively to hidden weight matrices.
This finding supports the common heuristic of excluding embeddings and 1D parameters when using Shampoo, challenging the view of Shampoo as a general-purpose improvement over Adam.

Overall, our observations indicate that Shampoo should not be viewed as an approximation of full-matrix Adam or related matrices (e.g., gradient covariance \citep{yang2008principal,sohldickstein2012natural,ida2017adaptive}, empirical Fisher \citep{kunstner2019limitations}, or Fisher/generalized Gauss-Newton matrices \citep{amari1998natural,martens2014new}).

\begin{tcolorbox}[
    enhanced,
    colback=white,       %
    colframe=black,      %
    coltitle=black,      %
    title=\textbf{Takeaway \#2},
    sharp corners,       %
    boxrule=0.8pt,       %
    bottom=4mm,
    attach boxed title to top left={xshift=0.5cm, yshift*=-0.8\baselineskip},
    boxed title style={
        frame hidden,    %
        colback=white,   %
        left=1mm, right=1mm %
    }
]
Shampoo's benefits can be exclusively attributed to its operation on \emph{weight matrices}. Consequently, Shampoo's preconditioner should likely not be interpreted as an approximation to full-matrix Adam.
\end{tcolorbox}

An alternative shape-agnostic perspective is viewing Shampoo as approximately running Adafactor/Adam in the eigenbasis of Shampoo's preconditioner \citep{vyas2025soap}, but it remains unclear why this basis is superior to others \citep{xie2025adam,maes2025understanding}.
Furthermore, Shampoo may outperform EShampoo; see \Cref{tab:eshampoo}.

\subsection{One-Sided Preconditioning}

Several works argue that two-sided preconditioning used in methods like K-FAC \citep{martens2015optimizing,grosse2016kronecker,martens2018kroneckerfactored,eschenhagen2023kfac} and Shampoo is inferior to one-sided preconditioning:
\citet{benzing2022gradient} argues that K-FAC approximates ``gradient descent on neurons", which corresponds to preconditioning solely with the damped, uncentered covariance of the activations (related to $\mR_t$ in Shampoo; see Appendix B in \citet{anil2020scalable}).
Similarly, \citet{xie2025structured} and \citet{an2025asgo} argue that one-sided Shampoo may be preferable to full-matrix or two-sided versions, citing tighter regret bounds in online convex settings.

\begin{table}[ht]
\caption{One- and two-sided KL-Shampoo on the Llama-320M model, $1 \times$ token budget, and $256$ batch size, and the rest of the setting as in \Cref{tab:core-methods}. Note that restricting KL-Shampoo to one side effectively simplifies into standard one-sided Shampoo$^{\sfrac{1}{2}}$.}
\label{tab:one-sided-shampoo}
\centering
\begin{tabular}{Gll}
    \toprule
    KL-Shampoo & $\mL_t$ Only & $\mR_t$ Only \\
    \hline  
    $24.95 \pm 0.09$ & $26.08 \pm 0.09$ & $25.42 \pm 0.07$ \\
    \bottomrule
\end{tabular}
\end{table}

Preconditioning only with $\mR_t$ performs significantly better than preconditioning only with $\mL_t$, consistent with the interpretation in \citet{benzing2022gradient}, as well as the experimental results in \citet{an2025asgo} and \citet{frans2025really} for SOAP.
However, neither one-sided variant matches the performance of the two-sided preconditioner. %
In fact, preconditioning only with $\mL_t$ performs worse than Muon (SVD).
\begin{tcolorbox}[
    enhanced,
    colback=white,       %
    colframe=black,      %
    coltitle=black,      %
    title=\textbf{Takeaway \#3},
    sharp corners,       %
    boxrule=0.8pt,       %
    bottom=4mm,
    attach boxed title to top left={xshift=0.5cm, yshift*=-0.8\baselineskip},
    boxed title style={
        frame hidden,    %
        colback=white,   %
        left=1mm, right=1mm %
    }
]
Only preconditioning with $\mR_t$ appears superior to $\mL_t$, but one-sided preconditioning fails to match the token efficiency of two-sided preconditioning.
\end{tcolorbox}
Given our findings, understanding Shampoo's two-sided preconditioning of weight matrices is necessary and sufficient to understand Shampoo in general.

\section{Characterizing Adaptation in Shampoo}
\label{sec:adaptation}

As the primary mechanism distinguishing Shampoo from spectral descent (\cf \Cref{eq:shampoo-spectral}), we study the EMAs in Adam and Shampoo's preconditioners. 
They aggregate gradients computed across different \emph{parameter iterates} and, in the stochastic setting, different \emph{mini-batches}, \eg
\begin{equation}
\label{eq:adam_ema_disentangled}
\begin{split}
    \vv_T &= \beta_2 \vv_{T-1} + \left( 1 - \beta_2 \right) \nabla \ell_{\mathcolor{BurntOrange}{\gB_T}}(\mathcolor{BurntOrange}{\vtheta_T})^2. \\
\end{split}
\end{equation}
We argue that adaptation in Shampoo operates along two distinct axes: the \emph{parameter trajectory} (averaging across iterates $\vtheta_t$) and \emph{stochasticity} (averaging across mini-batches $\gB_T$).\footnote{For simplicity, we focus our presentation on Adam's preconditioner, though the arguments extend analogously to Shampoo.}
To our knowledge, \citet{cattaneo2025memory} is the only study explicitly modeling both axes, arguing that this adaptation corresponds to modifying the loss function.

\subsection{Parameter Trajectory} 

Adam's preconditioner (\Cref{eq:adam_ema_disentangled}) can be viewed as approximating the EMA of the expected squared gradient:
\begin{equation}
\begin{split}
    \vv_T &\approx \beta_2 \vv_{T-1} + \left( 1 - \beta_2 \right) \E_{\gB \sim \gD}\left[ \nabla \ell_{\gB} (\vtheta_T)^2 | \vtheta_T \right] \\
    &=: \mathrm{EMA}_{t=1}^{T} \left[ \E_{\gB \sim \gD} \left[ \nabla \ell_{\gB} ( \vtheta_t )^2 \right] \right].
\end{split}
\end{equation}
This formulation represents a \emph{time-average} given the sequence of iterates $\gP = \{\vtheta_t\}_{t = 1}^T$, where the effective window size is governed by $\beta_2$.\footnote{We suppress the dependence on $\gP$ and $\beta_2$ in this notation.}
In the full-batch (deterministic) setting, this description matches Adam's preconditioner.

\begin{figure}[t]
\begin{center}
    \includegraphics[width=\columnwidth]{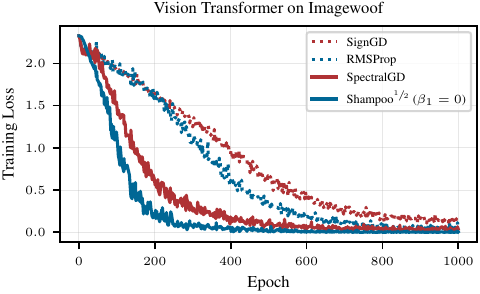}
\end{center}
\caption{\textbf{Full-batch setting.} RMSProp and Shampoo$^{\sfrac{1}{2}}$ only adapt to the parameter trajectory through the EMA in their preconditioner, but outperform SignGD and SpectralGD, respectively.}
\label{fig:full-batch}
\end{figure}

When comparing full-batch SignGD and SpectralGD to RMSProp and Shampoo$^{\sfrac{1}{2}}$ w/o momentum ($\beta_1=0$) on a subset of the Imagewoof dataset, we observe that the adaptive methods converge faster than their non-adaptive counterparts; see \Cref{fig:full-batch}.
This indicates that adapting to the parameter trajectory via time-averaging in the preconditioner is beneficial even in the absence of noise, suggesting that a perspective purely based on stochasticity is insufficient to explain the effectiveness of adaptive methods.
The mechanisms by which time-averaging confers these benefits remain actively researched \citep{cohen2025understanding,xie2025adam,xie2025tale}.

\subsection{Stochasticity}
\label{sec:stochasticity}

While adapting to the parameter trajectory is one important axis for adaptation, multiple common perspectives focus on the \emph{stochasticity} of the gradient estimator, viewing Adam's and Shampoo's preconditioners as approximations of gradient statistics at the current iterate.
For example, it is often assumed that Adam estimates the first and second moment of the gradient:
\begin{equation}
\label{eq:stochasticity-assumption}
    \vm_t \approx \E_{\gB \sim \gD}[\nabla \ell_{\gB} (\vtheta_t)], \hspace{1em} \vv_t \approx \E_{\gB \sim \gD}[\nabla \ell_{\gB} (\vtheta_t)^2].
\end{equation}
In this subsection, we investigate the limitations of two common formulations of adaptation to stochasticity.

\subsubsection{Variance Adaptation}
\label{sec:variance-adaptation}

As shown in \Cref{eq:adam-decomposition}, Adam can be interpreted as Signum with an element-wise variance adaptation scaling.
\citet{orvieto2025search} later expand upon this by showing that when $\beta_1 = \beta_2$, Adam's EMAs correspond to estimates of the mean and variance derived via variational inference, implicitly defining an adaptive trust region for Signum.

While ``variance adaptation'' is often discussed loosely, \citet{balles2017dissecting} propose a precise definition: %
\begin{proposition}[restate=optimaladaptationgeneral, name=Lemma 1 in \citet{balles2017dissecting}]
\label{prop:dissecting_adam}
    Let $\vg \in \R^d$ be a random variable with $\E[\vg] = \nabla \gL$ and $\mathrm{Var}(\vg) = \sigma^2$. Then $\E[|| \bm{\gamma} \odot \vg - \nabla \gL||_2^2]$ is minimized by
    \begin{equation}
    \label{eq:optimal-adaptation-gradient}
        \gamma_i = \frac{\E[g_i]^2}{\E[g_i^2]} = \frac{\nabla \gL_i^2}{\nabla \gL_i^2 + \sigma_i^2} = \frac{1}{1 + \frac{\sigma_i^2}{\nabla \gL_i^2}}
    \end{equation}
    and $\E[|| \bm{\gamma} \odot \mathrm{sign}(\vg) - \mathrm{sign}(\nabla \gL)||_2^2]$ is minimized by
    \begin{equation}
    \label{eq:optimal-sign-adaptation}
        \gamma_i = 2 \; \mathrm{P}(\mathrm{sign}(g_i) = \mathrm{sign}(\nabla \gL_i)) - 1.
    \end{equation}
\end{proposition}
Alternatively, one could minimize an expected distance between the element-wise scaled stochastic gradient and the sign of the deterministic gradient directly:
\begin{proposition}[restate=optimaladaptationsign, name=]
\label{prop:alternative_dissecting_adam}
    In the setting from \Cref{prop:dissecting_adam}, $\E \left[ || \bm{\gamma} \odot \vg - \mathrm{sign}\left( \nabla \gL \right) ||_2^2 \right]$ is minimized by
    \begin{equation}
    \label{eq:optimal-variance-adaptation-sign}
       \gamma_i = \frac{\nabla \gL_i \mathrm{sign}\left( \nabla \gL_i \right) }{\E \left[ g_i^2 \right]} = \frac{|\nabla \gL_i|}{\E \left[ g_i^2 \right]}.
    \end{equation}
\end{proposition}

Naturally, we can extend this analysis to Shampoo by considering matrix preconditioning instead of element-wise scaling.
For simplicity, we focus on the one-sided case where we fix $\mB = \mI$, though the derivation generalizes to the two-sided case; see \Cref{sec:notes-variance-adaptation}.
\begin{proposition}[restate=optimaladaptationspectral, name=]
\label{prop:one-sided-optimal-adaptation-spectral}
    Let $\mG \in \R^{m \times n}$ be a random variable with $\E[\mG] = \nabla \gL$ and finite covariance. Let $\nabla \gL = \mU \mSigma \mV^\transpose$ be the SVD of the gradient. When $\mB = \mathbf{I}$, $\E[||\mA \mG \mB - \mU \mV^\transpose||_F^2]$ is minimized by
    \begin{equation}
    \label{eq:optimal-A-spectral-update}
    \begin{split}
        \mA &= \left( \nabla \gL \nabla \gL^\transpose \right)^{ \frac{1}{2}} \E\left[ \mG \mG^\transpose \right]^{-1}. \\
    \end{split}
    \end{equation}
\end{proposition}

However, in all cases, we encounter two critical discrepancies between this theoretical formalization and Adam's practical implementation, and analogously for Shampoo.

\textbf{Dependence on the oracle gradient.} \Cref{eq:optimal-sign-adaptation,eq:optimal-variance-adaptation-sign} both require access to the deterministic gradient $\nabla \gL$. 
\citet{balles2017dissecting} propose using the EMA $\vm$ as an estimator of $\nabla \gL$ for variance adaptation, but it is unclear why this is preferable to using $\sign(\vm) \approx \sign(\nabla \gL)$.

\textbf{A remaining gap to the update.} Even if we accept the substitution $\vm$ for $\nabla \gL$, \citet{balles2017dissecting} only show that under Gaussian noise, the optimal sign adaptation factor in \Cref{eq:optimal-sign-adaptation} empirically tracks the idealized variance adaptation factor.
When we instead consider the optimal solution in \Cref{eq:optimal-variance-adaptation-sign} and adapt $\vm$ instead of $\vg$, we have
\begin{equation}
\label{eq:optimal-variance-approx}
    \gamma \odot \vm = \mathcolor{MidnightBlue}{\left( \vm^2 \oslash \E\left[ \vg^2 \right] \right)} \odot \mathcolor{Maroon}{\mathrm{sign}\left( \vm \right)}.
\end{equation}
Contrast this with the idealized Adam update,
\begin{equation}
\label{eq:adam-in-terms-of-optimal}
\begin{split}
    \gamma^\mathrm{Adam} \odot \vm
    &= \mathcolor{MidnightBlue}{\left( \vm^2 \oslash \E\left[ \vg^2 \right] \right)}^{\textcolor{BurntOrange}{\frac{1}{2}}} \odot \mathcolor{Maroon}{\mathrm{sign}\left( \vm \right)},
\end{split}
\end{equation}
where we have an additional \textcolor{BurntOrange}{square root} of the scaling.
Regardless of the definition, we cannot recover Shampoo and Adam's update from this form of variance adaptation.

\subsubsection{Whitening}
\label{sec:whitening}

Whitening is another approach to handling stochasticity that is commonly cited as motivation for Shampoo \citep{vyas2025improving,frans2025really}.
In stochastic optimization, it is typically defined as a linear transform that maps the stochastic gradient to a random vector with isotropic covariance \citep{yang2008principal} (\cf \Cref{def:whitening-vector}).
Our experiments in \Cref{sec:beyond} suggest that this standard  whitening is not a meaningful interpretation of Shampoo.
However, Shampoo aligns closely with a form of \emph{matrix whitening}:
\begin{definition}[restate=matrixwhitening, name=]
\label{def:whitening-matrix}
    Let $\mG \in \R^{m \times n}$ be a random matrix with $\E\left[ \mG \right] = \nabla \gL$ and finite row- and column-wise covariance:
    \begin{equation*}
        \begin{aligned}
            \mathrm{Cov}_\mathrm{row}\left( \mG \right) & = \E\left[ (\mG - \nabla \gL) (\mG - \nabla \gL)^\transpose \right] \\
            \mathrm{Cov}_\mathrm{col}\left( \mG \right) & = \E\left[ (\mG - \nabla \gL)^\transpose (\mG - \nabla \gL)  \right].
        \end{aligned}
    \end{equation*}
    We call symmetric positive-definite matrices $\mA \in \R^{m \times m}$ and $\mB \in \R^{n \times n}$ \emph{matrix whitening matrices} if $\mathrm{Cov}_\mathrm{row}\left( \mA \mG \mB \right) = \mI_m$ and $\mathrm{Cov}_\mathrm{col}\left( \mA \mG \mB \right) = \mI_n$.
\end{definition}

\begin{corollary}[restate=whitening, name=]
\label{cor:full-whitening}
    The optimality condition for whitening $\mG$ according to \Cref{def:whitening-matrix} is
    \begin{equation}
        \mA = \mathrm{Cov}_\mathrm{row}\left( \mG \mB \right)^{-\frac{1}{2}}, \hspace{1em} \mB = \mathrm{Cov}_\mathrm{col}\left( \mA \mG \right)^{-\frac{1}{2}}.
    \end{equation}
\end{corollary}

When replacing the centered with the uncentered covariance, \Cref{cor:full-whitening} is equivalent to the optimality condition proposed in Section 3.1 of \citet{vyas2025improving} and to the KL-divergence minimization problem in \Cref{eq:kl-divergence} \citep{lin2025understanding}.
KL-Shampoo can thus be interpreted as a practical algorithm for inexactly solving this optimality condition.

However, the matrix whitening formalized in \Cref{def:whitening-matrix} raises two critical questions: 1) why is row- and column-wise whitening superior to full whitening of the gradient vector, and 2) why should the uncentered covariance be used instead of the centered covariance?
While the latter question has been investigated for Adam variants by using a diagonal approximation of the centered covariance \citep{graves2014generating,ida2017adaptive,zhuang2020adabelief}, it remains unclear whether centered approaches are beneficial \citep{schmidt2021descending}.
See \Cref{sec:centered-shampoo-appendix} for a centered variant of Shampoo.

These questions suggest that an interpretation with similar structural characteristics, but less conceptual baggage, could better describe Shampoo's inner workings.

\subsection{Adapting Spectral Descent}
\label{sec:description}

To unify Shampoo's adaptation to stochasticity and the parameter trajectory with spectral descent's strict semi-orthogonality constraint, we define:
\begin{definition}[restate=taoe, name=Time-averaged orthogonality in expectation]
\label{def:time-averaged-orthogonality}
    Let $\gG = \{ \mG_t \}_{t=1}^T$ be a sequence of random variables $\mG_t: \Omega \rightarrow \R^{m \times n}$ with sample space $\Omega$. We call $\gG$ \emph{time-averaged orthogonal in expectation} if we have
    \begin{equation}
    \mathcolor{MidnightBlue}{\mathrm{EMA}_{t=1}^T} \left({\mathcolor{MidnightBlue}{\E}[\mathcolor{Maroon}{\mG_t \mG_t^\transpose}]} \right) = \mathcolor{Maroon}{\mathbf{I}_m}, \; \mathcolor{MidnightBlue}{\mathrm{EMA}_{t=1}^T} \left( {\mathcolor{MidnightBlue}{\E}[\mathcolor{Maroon}{\mG_t^\transpose \mG_t}]} \right) = \mathcolor{Maroon}{\mathbf{I}_n}.
    \end{equation}
\end{definition}
\begin{corollary}[restate=idealklshampoo, name=Idealized KL-Shampoo]
\label{cor:ideal-kl-shampoo}
    Let $\gZ = \{ \mZ_t \}_{t=1}^T$ with $\mZ_t = \mA \mG_t \mB$ and with symmetric positive definite matrices $\mA$ and $\mB$. For $\gZ$ to be time-averaged orthogonal in expectation, $\mA$ and $\mB$ have to fulfill the optimality condition
    \begin{equation}
    \label{eq:idealized-kl-shampoo}
    \begin{aligned}
        \mA &= \mathrm{EMA}_{t=1}^T \left(\E\left[ \mG_t \mB^{2} \mG_t^\transpose \right] \right)^{-\frac{1}{2}}, \\
        \mB &= \mathrm{EMA}_{t=1}^T \left( \E\left[ \mG_t^\transpose \mA^{2} \mG_t \right] \right)^{-\frac{1}{2}}.
    \end{aligned}
    \end{equation}
\end{corollary}
Solving \Cref{eq:idealized-kl-shampoo} exactly is infeasible in practice due to the coupling of $\mA$ and $\mB$, unlike the analogous condition for RMSProp; see \Cref{sec:shampoo-adapted-optimality}.
It corresponds to a time-average of the type of whitening discussed in \Cref{sec:whitening}, but relates it to the orthogonality constraint in spectral descent ($\mathcolor{Maroon}{\bullet}$), adapted to stochasticity and the parameter trajectory ($\mathcolor{MidnightBlue}{\bullet}$).
In contrast to notions of variance adaptation and whitening (\cf \Cref{sec:stochasticity}), this perspective offers a more complete description of the actual algorithm, consistent with our empirical results.

Without the time-average and stochasticity, \Cref{def:time-averaged-orthogonality} reduces to standard semi-orthogonality.
However, it is not obvious how KL-Shampoo could reduce to spectral descent.
\begin{proposition}[restate=instantklshampoo, name=``Instantaneous" KL-Shampoo converges to spectral descent.]
\label{prop:instant-kl-shampoo}
    Consider the iteration in \Cref{eq:kl-shampoo} initialized with $\mL_0 = \mR_0 = c\mathbf{I}$ for any $c \in \R^+$, and a fixed $\mG$ with reduced SVD $\mG = \mU \mSigma \mV^\transpose$.
    Then we have
    \begin{equation}
        \mL_\infty^{\frac{\dagger}{2}} \mG \mR_\infty^{\frac{\dagger}{2}} = \mU \mV^\transpose,
    \end{equation}
    where $\mL_\infty : = \lim_{t \rightarrow \infty} \mL_t$ and $\mR_\infty : = \lim_{t \rightarrow \infty} \mR_t$.
\end{proposition}
In this specific setting, KL-Shampoo recovers spectral descent analogously to \Cref{eq:shampoo-spectral}, despite using $p=1/2$.
As a result, one potential hypothesis for why Shampoo$^{\sfrac{1}{2}}$ outperforms Shampoo$^{\sfrac{1}{4}}$ is that Shampoo$^{\sfrac{1}{2}}$ is a tighter approximation to KL-Shampoo than Shampoo$^{\sfrac{1}{4}}$.

\begin{tcolorbox}[
    enhanced,
    colback=white,       %
    colframe=black,      %
    coltitle=black,      %
    title=\textbf{Takeaway \#4},
    sharp corners,       %
    boxrule=0.8pt,       %
    bottom=4mm,
    attach boxed title to top left={xshift=0.5cm, yshift*=-0.8\baselineskip},
    boxed title style={
        frame hidden,    %
        colback=white,   %
        left=1mm, right=1mm %
    }
]
Just as spectral descent enforces semi-orthogonal updates, KL-Shampoo enforces \emph{time-averaged semi-orthogonal updates in expectation}.
``Instantaneous" KL-Shampoo converges to spectral descent.
\end{tcolorbox}

\section{Discussion and Conclusion}
\label{sec:conclusion}
While this work has characterized how Shampoo adapts to stochasticity and the parameter trajectory, disentangling the effects of this theoretically and empirically remains underexplored, even for Adam.
Several lines of work appear relevant, including central flows \citep{cohen2025understanding}, adaptive smoothness \citep{xie2025adam,xie2025tale}, implicit loss modification \citep{cattaneo2025memory}, and Bayesian filtering \citep{aitchison2020bayesian}.
Concretely, one could explore using $L^p$ norms and time-averaging to build towards an adapted modular duality theory (\cf \Cref{sec:adapting-descent}).

Besides adaptation, another fundamental gap remains: why does optimization in the $S_\infty$ geometry (via SpectralGD) outperform $\ell_\infty$ geometry (via SignGD) and $\ell_2$ geometry (via GD) in deep learning \citep{balles2020geometry,su2025isotropic,davis2025spectral}?
While recent work attributes this partly to robustness to heavy-tailed class imbalance \citep{kunstner2024heavytailed,kunstner2025scalinglaws,wang2025muon}, its benefits seem to extend beyond this setting.

Practically, broad adoption requires reducing Shampoo's memory, computational, communication, and implementation overhead.
Beyond efficiency, it is critical to understand the distinct training dynamics and learned solutions of matrix optimizers.
For instance, does Shampoo suffer from the same large-scale stability issues as Muon \citep{kimiteam2025kimik2}?
How do Shampoo and Muon affect downstream applications such as continual learning and quantization \citep{pascanu2025optimizers,he2024understanding,vlassis2025outliers}?

\section*{Acknowledgements}

We thank Gavin (Jialun) Zhang for his in-depth feedback on our initial manuscript and Hiroki Naganuma for his work on the early experiment infrastructure.
We also thank Jeremy Cohen, Felix Dangel, and Wu Lin for various insightful discussions.
We appreciate the consistent advice and managerial support from Adnan Aziz, Jana van Greunen, Amit Nagpal, Maxim Naumov, Sandeep Parab, Joel Pobar, Chunqiang Tang, and Richard E. Turner.

Runa Eschenhagen is supported by ARM, the Cambridge Trust, and the Qualcomm Innovation Fellowship.
This work was performed by Runa while he was an intern in the Meta AI \& Systems Co-Design team and external research collaborator in the Meta Superintelligence Labs Infrastructure Kernels \& Optimizations team.

\section*{Impact Statement}

This paper presents work whose goal is to advance the field of Machine
Learning. There are many potential societal consequences of our work, none
which we feel must be specifically highlighted here.

\bibliography{bibliography}
\bibliographystyle{icml2026}

\newpage
\appendix
\onecolumn
\section{Algorithms}
\label{sec:algos-appendix}

\begin{table}[t]
\caption{Table showing the relationship between different optimizers for choices of $\beta_1$ and $\beta_2$, similar to Table 1 in \citet{ziyin2020laprop}. 
This captures the relationships more precisely than \Cref{fig:summary} (middle). 
We use $:=$ to indicate that this setting corresponds to the general definition of the optimizer in question. 
We ignore numerical differences between methods for computing the sign, orthogonalization, and preconditioning. 
Here, momentum refers to the EMA of the (preconditioned) gradient, controlled by $\beta_1$.}
\label{tab:optimizer-variants-betas}
\centering
\begin{tabular}{llll}
\toprule
\textbf{Optimizer} & \multicolumn{3}{c}{\textbf{Hyperparameters}} \\
& & $\beta_2 = 0$ & $\beta_2 > 0$ \\
\midrule
\multirow{2}{*}{Adam \citep{kingma2014adam}}
& $\beta_1 = 0$ & SignGD & RMSProp \\
& $\beta_1 > 0$ & \textcolor{BurntOrange}{unstable} & := \\
\midrule
\multirow{2}{*}{LaProp \citep{ziyin2020laprop}}
& $\beta_1 = 0$ & SignGD & RMSProp \\
& $\beta_1 > 0$ & SignGD w/ momentum & := \\
\midrule
\multirow{2}{*}{BCOS-m \citep{jiang2025stochastic}}
& $\beta_1 = 0$ & SignGD & RMSProp \\
& $\beta_1 > 0$ & Signum & := \\
\midrule
\multirow{2}{*}{\makecell[l]{Shampoo$^{\sfrac{1}{4}}$ \\\citep{gupta2018shampoo,shi2023distributed}}}
& $\beta_1 = 0$ & SpectralGD & Shampoo w/o momentum \\
& $\beta_1 > 0$ & \textcolor{BurntOrange}{unstable} & := \\
\midrule
\multirow{2}{*}{LaProp Shampoo$^{\sfrac{1}{4}}$ (\Cref{sec:laprop-shampoo})}
& $\beta_1 = 0$ & SpectralGD & Shampoo w/o momentum \\
& $\beta_1 > 0$ & SpectralGD w/ momentum & := \\
\midrule
\multirow{2}{*}{BCOS-m Shampoo$^{\sfrac{1}{4}}$ (\Cref{sec:bcosm-shampoo})}
& $\beta_1 = 0$ & SpectralGD & Shampoo w/o momentum \\
& $\beta_1 > 0$ & Muon & := \\
\bottomrule
\end{tabular}
\end{table}

\begin{algorithm}[ht]
    \caption{Meta-Optimizer Pseudocode.}
    \label{alg:algo-template}
    \begin{algorithmic}[1]
        \Require Parameter $\mW_1 \in \R^{m \times n}$ with $\vtheta_1 = \vectorized{\left( \mW_1 \right)} \in \R^{mn}$, learning rates $\alpha_t > 0$, exponent $p > 0$, weight decay $\lambda \geq 0$, $\epsilon > 0$, $\beta_1, \beta_2, \beta_3 \in [0, 1)$, and bias corrections $c_{i, t} \in \{1,  1 - \beta_i^t\}$.
        \State Initialize $\mC_0 \in \R^{mn \times mn}, \vm_0 = \vv_0 = \mathbf{0} \in \R^{mn}$.
        \For{$t = 1, ..., T$}
            \State Sample $\mathcal{B} \sim \mathcal{D}$.
            \State Compute $\vg_t = \nabla \ell_{\mathcal{B}} (\vtheta_t)$.
            \State Update $\vm_t = \beta_1 \vm_{t-1} + \left( 1 - \beta_1 \right) \vg_t$.
            \State Update $\mC_t = \mathtt{update\_preconditioner}\left(\mC_{t-1}, \beta_2, c_{2, t}, \vg_t, \vm_t / c_{1, t} \right)$.
            \State Compute $\vu_t = \mathtt{grafting}\left(\mathtt{stabilized\_inverse\_root}\left(\mC_t, \epsilon, p \right) \vm_t / c_1^{(t)} \right)$.
            \State Update $\vv_t = \beta_3 \vv_{t-1} + \left( 1 - \beta_3 \right) \vu_t$.
            \State Update $\vtheta_{t + 1} = \vtheta_t - \alpha_t (\vv_t / c_3^{(t)} + \lambda \vtheta_t)$.
        \EndFor
    \end{algorithmic}
\end{algorithm}

\subsection{Shampoo Variants}

See \Cref{tab:optimizer-variants-betas} for an overview of the relationship of Adam and Shampoo variants with sign and spectral descent variants, depending on $\beta_1$ and $\beta_2$.
A less precise version of this table is also visualized in \Cref{fig:summary} (middle). All considered algorithms (including all algorithms in \Cref{tab:optimizer-variants-betas}, KL-Shampoo, and EShampoo variants) can be implemented within \Cref{alg:algo-template} by choosing the corresponding subroutines.

\subsubsection{LaProp Shampoo}
\label{sec:laprop-shampoo}

Analogously to LaProp \citep{ziyin2020laprop}, we can define a Shampoo update where we first precondition the stochastic gradient and then compute the EMA over the \emph{preconditioned} gradient, \ie $\beta_1=0$ and $\beta_3 \in (0, 1)$ in \Cref{alg:algo-template}.

\subsubsection{BCOS-m Shampoo}
\label{sec:bcosm-shampoo}

Analogously to BCOS-m \citep{jiang2025stochastic}, we can define BCOS-m Shampoo, with factor matrices:
\begin{equation}
\begin{split}
    \mL_t &= \beta_2 \mL_{t-1} + (1 - \beta_2) \mM_t \mM_t^\transpose, \\
    \mR_t &= \beta_2 \mR_{t-1} + (1 - \beta_2) \mM_t^\transpose \mM_t, \\
\end{split}
\end{equation}
or BCOS-m KL-Shampoo, with factor matrices:
\begin{equation}
\begin{split}
    \mL_t &= \beta_2 \mL_{t-1} + (1 - \beta_2) \mM_t \mR_{t-1}^{-1} \mM_t^\transpose, \\
    \mR_t &= \beta_2 \mR_{t-1} + (1 - \beta_2) \mM_t^\transpose \mL_{t-1}^{-1} \mM_t, \\
\end{split}
\end{equation}
where $\mM_t = \beta_1 \mM_{t-1} + (1-\beta_1) \mG_t$.
We then precondition $\mM_t$ with the inverse roots of $\mL_t$ and $\mR_t$.

\subsubsection{Centered Shampoo}
\label{sec:centered-shampoo-appendix}

In its simplest form, centered Shampoo's preconditioner could be defined as
\begin{equation}
\begin{split}
    \mL_t &= \beta_2 \mL_{t-1} + (1 - \beta_2) \left( \mG_t \mG_t^\transpose - \mM_t \mM_t^\transpose \right), \\
    \mR_t &= \beta_2 \mR_{t-1} + (1 - \beta_2) \left( \mG_t^\transpose \mG_t - \mM_t^\transpose \mM_t \right), \\
\end{split}
\end{equation}
or alternatively by preconditioning with
\begin{equation}
\begin{split}
    \tilde{\mL_t} &= \mL_t - \mM_t \mM_t^\transpose, \\
    \tilde{\mR_t} &= \mR_t - \mM_t^\transpose \mM_t, \\
\end{split}
\end{equation}
where $\mL_t$ and $\mR_t$ are the regular Shampoo factor matrices.
We can add $\epsilon$ to both variants for stability and to ensure that each term is positive semi-definite.

\subsection{Implementation of KL-Shampoo}
\label{sec:kl-shampoo-implementation}

The KL-Shampoo update in \Cref{eq:kl-shampoo} is implemented via outer products of preconditioned gradient matrices,
\begin{equation}
\begin{split}
    \tilde{\mG}_t^\mL &= \mG_t (\mR_{t-1} + \epsilon \mathbf{I}_n)^{-\frac{1}{2}} \\
    \tilde{\mG}_t^\mR &= (\mL_{t-1} + \epsilon \mathbf{I}_m)^{-\frac{1}{2}}\mG_t \\
    \mL_t &= \beta_2 \mL_{t-1} + (1 - \beta_2) \tilde{\mG}_t^\mL \tilde{\mG}_t^{\mL\transpose}, \\
    \mR_t &= \beta_2 \mR_{t-1} + (1 - \beta_2) \tilde{\mG}_t^{\mR\transpose} \tilde{\mG}_t^\mR, \\
\end{split}
\end{equation}
where the (potentially stale) inverse roots can be reused from previous iterations.
This highlights the more complex role of $\epsilon$ in KL-Shampoo compared to regular Shampoo, where it is only used for preconditioning in the update (\cf \Cref{eq:shampoo}). 

The original KL-Shampoo algorithm proposed by \citet{lin2025understanding} also introduces a dimension-dependent scaling and a per-factor eigenvalue correction.
Since these modifications could be applied to all Shampoo variants and we want to isolate the coupled factor matrices update in KL-Shampoo, we do not use these modifications.
However, they are necessary when updating the preconditioner sparsely and no grafting is used \citep{eschenhagen2025purifying}.

\section{Experimental Details}
\label{sec:experiments-appendix}

\subsection{Language Modeling}
\label{sec:language-experiments}

\subsubsection{Experimental Setup}

All language models are trained on the \texttt{TorchTitan} stack \cite{liang2025torchtitan} using Llama 3 architectures \citep{grattafiori2024llama3} and C4 data \cite{raffel2023exploring}; configuration details can be found in \cref{tab:model-config} below.
All optimizers are implemented in the PyTorch Distributed Shampoo codebase \citep{shi2023distributed} (\Cref{alg:algo-template}).

\begin{table}[t]
    \caption{Configuration details for each model size used in our experiments. The number of tokens was originally determined by applying the Chinchilla optimal heuristic ($1\times \approx 20$ tokens per parameter) on the total number of parameters; however, we based our calculations on incorrect total parameter counts. As prior work suggests that compute optimality may be significantly lower than $20$ times the parameter count for matrix optimizers \citep{chen2025scale}, we decided to keep our original token budgets.}
    \label{tab:model-config}
    \centering
    \begin{tabular}{l cc}
    \toprule
         $P$ & 320M & 1.5B \\
         \midrule
         dimension & 768 & 1792 \\
         number of layers & 18 & 26 \\
         heads & 12 & 16 \\
         sequence length & 2048 & 2048 \\
         vocabulary size & 128K & 128K \\
         \midrule
         number of tokens in $1\times$ budget & 3.2B & 22B \\
    \bottomrule
    \end{tabular}
\end{table}

We categorize the four settings into ``small" runs -- the 320M $1\times$ token budget settings (i.e. the first two rows in \cref{tab:core-methods}) -- and ``large" runs -- the 320M $8\times$ token budget and 1.5B $1\times$ token budget settings (i.e. the last two rows in \cref{tab:core-methods}). All settings follow the same procedure for hyperparameter tuning described below, with minor differences between ``small" and ``large" settings.

\begin{enumerate}
    \item We initially sweep learning rate $\alpha$ and EMA parameters ($\beta_1$, $\beta_2$), using the following default values for all other hyperparameters and training settings:
    \begin{itemize}
        \item epsilon $\epsilon$: $10^{-8}$ for AdamW, $10^{-12}$ for Shampoo methods
        \item weight decay: 0.1\\
        We did sweep the values 0.02 and 0.5 for Muon and KL-Shampoo (in the 320M, $1\times$ token budget, batch size 256 setting) as well, confirming that for both methods the order was the same and 0.1 performed best.
        \item learning rate schedule: 10\% linear warmup followed by cosine decay to 0.0
        \item gradient clipping: global maximum $\ell_2$ norm of 1.0
        \item precondition frequency: 1 (every iteration) for Shampoo methods
        \item fixed seed, but with deterministic algorithms turned off
    \end{itemize}
    We determine initial search spaces for learning rate and $\beta$s based on the study by \citet{orvieto2025search}, or for our ``large" settings, based on experiments in the ``small" settings. We extend the search space for each hyperparameter until a U-curve around each hyperparameter's best value is achieved. Final overall search spaces are reported in \cref{tab:hparams-space}.
    \item For AdamW and the Shampoo methods, we then sweep $\epsilon$ using the best learning rate and $\beta$s from the initial sweep.
    \item For methods that use (AdamW) grafting, we use the best $\beta_2$ and $\epsilon$ from the AdamW sweeps in the equivalent setting for the grafted optimizer.
    \item To compute the final reported results, we repeat the best hyperparameter configuration from the above sweeps for 10 different seeds. For these runs, we turn on deterministic algorithms.
\end{enumerate}

\begin{table}[t]
    \centering
    \caption{Final search spaces covered by hyperparameter sweeps. We report inclusive [mininum, maximum] values of each hyperparameter swept for each method/setting. For each range, we sweep all values in $2\times$ (learning rate, $1 - \beta_1$, $1 - \beta_2$) or $10\times$ ($\epsilon$) increments. (Single values indicate that the hyperparameter was not swept.) For instance, in the Shampoo$^{\sfrac{1}{4}}$ 320M $1\times$ batch size 64 setting (first row), the swept hyperparameters are \{0.002, 0.004, 0.08, 0.016\} for learning rate, \{0.8, 0.9, 0.95, 0.975, 0.9875\} for $\beta_1$, \{0.6, 0.8, 0.9, 0.95, 0.975, 0.9875, 0.99375, 0.996875\} for $\beta_2$, and \{$10^{-33}$, $10^{-32}$, $10^{-31}$, \dots, $10^{-14}$, $10^{-13}$, $10^{-12}$\} for $\epsilon$.}
    \label{tab:hparams-space}
    \begin{tabular}{lll l cccc}
\toprule
$P$ & $T$ & $B$ & Method & $\alpha$ & $\beta_1$ & $\beta_2$ & $\epsilon$ \\
\midrule
\multirow{6}{*}{320M} & \multirow{6}{*}{$1\times$} & \multirow{6}{*}{64}
  & Shampoo$^{\sfrac{1}{4}}$   & [0.002, 0.016] & [0.8, 0.9875] & [0.6, 0.996875] & [$10^{-33}$, $10^{-12}$] \\
  &&& Shampoo$^{\sfrac{1}{2}}$   & [0.002, 0.016] & [0.8, 0.9875] & [0.6, 0.996875] & [$10^{-16}$, $10^{-12}$] \\
  &&& KL-Shampoo    & [0.002, 0.016] & [0.8, 0.9875] & [0.6, 0.996875] & [$10^{-13}$, $10^{-6}$] \\
  &&& Muon          & [0.002, 0.016] & [0.8, 0.9875] & -- & -- \\
  &&& AdamW         & [0.001, 0.016] & [0.8, 0.99375] & [0.6, 0.996875] & $10^{-8}$ \\
  &&& Signum        & [0.0000625, 0.004] & [0.8, 0.9875] & -- & -- \\
\midrule
\multirow{6}{*}{320M} & \multirow{6}{*}{$1\times$} & \multirow{6}{*}{256}
  & Shampoo$^{\sfrac{1}{4}}$   & [0.008, 0.064] & [0.8, 0.9875] & [0.6, 0.996875] & [$10^{-25}$, $10^{-10}$] \\
  &&& Shampoo$^{\sfrac{1}{2}}$   & [0.008, 0.064] & [0.8, 0.9875] & [0.6, 0.996875] & [$10^{-18}$, $10^{-8}$] \\
  &&& KL-Shampoo    & [0.008, 0.064] & [0.8, 0.9875] & [0.6, 0.996875] & [$10^{-12}$, $10^{-6}$] \\
  &&& Muon          & [0.008, 0.064] & [0.8, 0.9875] & -- & -- \\
  &&& AdamW         & [0.001, 0.016] & [0.8, 0.9875] & [0.6, 0.996875] & $10^{-8}$ \\
  &&& Signum        & [0.0000625, 0.004] & [0.8, 0.9875] & -- & -- \\
\midrule
\multirow{6}{*}{320M} & \multirow{6}{*}{$8\times$} & \multirow{6}{*}{256}
  & Shampoo$^{\sfrac{1}{4}}$   & [0.002, 0.016] & [0.95, 0.996875] & [0.8, 0.9875] & [$10^{-36}$, $10^{-12}$] \\
  &&& Shampoo$^{\sfrac{1}{2}}$   & [0.002, 0.016] & [0.95, 0.996875] & [0.8, 0.975] & [$10^{-16}$, $10^{-8}$] \\
  &&& KL-Shampoo    & [0.002, 0.008] & [0.9875, 0.996875] & [0.8, 0.975] & [$10^{-14}$, $10^{-6}$] \\
  &&& Muon          & [0.002, 0.016] & [0.9, 0.99375] & -- & -- \\
  &&& AdamW         & [0.002, 0.016] & [0.95, 0.99375] & [0.8, 0.975] & [$10^{-15}$, $10^{-5}$] \\
  &&& Signum        & [0.0005, 0.004] & [0.9, 0.9875] & -- & -- \\
\midrule
\multirow{6}{*}{1.5B} & \multirow{6}{*}{$1\times$} & \multirow{6}{*}{256}
  & Shampoo$^{\sfrac{1}{4}}$   & [0.002, 0.016] & [0.9, 0.996875] & [0.8, 0.9875] & [$10^{-36}$, $10^{-12}$] \\
  &&& Shampoo$^{\sfrac{1}{2}}$   & [0.001, 0.016] & [0.9, 0.996875] & [0.8, 0.975] & [$10^{-17}$, $10^{-12}$] \\
  &&& KL-Shampoo    & [0.001, 0.004] & [0.9875, 0.996875] & [0.8, 0.9875] & [$10^{-12}$, $10^{-2}$] \\
  &&& Muon          & [0.002, 0.016] & [0.9, 0.996875] & - & - \\
  &&& AdamW         & [0.001, 0.008] & [0.95, 0.996875] & [0.9, 0.996875] & [$10^{-16}$, $10^{-9}$] \\
  &&& Signum        & [0.005, 0.004] & [0.6, 0.95] & - & - \\
\bottomrule
\end{tabular}
\end{table}

Below we list the exceptions from the above procedure:
\begin{itemize}
    \item We did not tune $\epsilon$ for AdamW in the ``small" runs, so all methods with grafting use the default AdamW epsilon $10^{-8}$ in the grafted optimizer. We believe it is unlikely for a tuned grafting $\epsilon$ to affect different methods unequally and change the ordering of the results in \cref{tab:core-methods}.
    \item For ``large" runs of Shampoo$^{\sfrac{1}{4}}$, we swept $\epsilon$ in $1000\times$ increments instead of $10\times$ due to the large potential range of $\epsilon$.
    \item The best hyperparameter configuration from the initial sweep of Shampoo$^{\sfrac{1}{4}}$ in the 320M $8\times$ token budget setting proved to be unstable in the 10 repeat runs. So, we selected the second-best configuration from the sweep, which used $\beta_2 = 0.95$ instead of $0.9$. This configuration was still somewhat unstable, with 2 of 10 runs having significantly worse performance, causing the high reported variance in \cref{tab:core-methods}.
    \item In all settings, we assume Muon (SVD) and Muon (NS) to have the same best hyperparameter configuration, so to save resources, all sweeps were tuned with NS. Only 10 repeat runs on different seeds were run with SVD.
\end{itemize}

\subsubsection{Compute Costs}

\begin{table*}[h!]
\caption{Estimated compute-hours per run for each core method in each setting, with deterministic computation.}
\label{tab:compute-estimates}
\centering
\begin{scriptsize}
\begin{tabular}{lllccccccccc}
\toprule
$P$ & $T$ & $B$ & \# GPUs & GPU Type & Shampoo$^{\sfrac{1}{4}}$ & Shampoo$^{\sfrac{1}{2}}$ & KL-Shampoo & Muon (SVD) & Muon (NS) & AdamW & Signum \\
\hline
\multirow{3}{*}{320M} & \multirow{2}{*}{$1\times$} & 64 & 8 & H100 & 56 & 56 & 56 & 72 & 32 & 32 & 32 \\
& & 256 & 16 & H100 & 32 & 32 & 32 & 32 & 24 & 24 & 24 \\
& $8\times$ & 256 & 8 & GB200 & 152 & 152 & 152 & 184 & 112 & 72 & 72 \\
\hline
1.5B & $1\times$ & 256 & 8 & GB200 & 272 & 272 & 272 & 352 & 160 & 144 & 144 \\
\bottomrule
\end{tabular}
\end{scriptsize}
\end{table*}

We estimate the compute-hours per run for each method and setting in \cref{tab:compute-estimates} by multiplying the approximate total runtime and number of machines used. Based on these estimates, we estimate a total of 60,000 compute-hours on NVIDIA H100s and 127,000 compute-hours on NVIDIA GB200s across $\sim$2,300 runs were used to produce our results in \cref{tab:core-methods}. To produce the ablations in the rest of the paper, an additional 500 compute-hours on H100s and 3,000 compute-hours on GB200s across $\sim$1,600 runs were used. These estimates do not include any additional exploratory experiments we performed throughout our investigation.

\subsubsection{Hyperparameter Trends}

\begin{table}[t]
\centering
\caption{Best hyperparameters for each core method in each setting in \cref{tab:core-methods}. \textcolor{BurntOrange}{Colored} values were not tuned and starred ($^*$) values were manually adjusted; see text for more details on exceptions to our tuning procedure.}
\label{tab:hparams-core}
\begin{tabular}{lll l cccc}
\toprule
$P$ & $T$ & $B$ & Method & $\alpha$ & $\beta_1$ & $\beta_2$ & $\epsilon$ \\
\midrule
\multirow{6}{*}{320M} & \multirow{6}{*}{$1\times$} & \multirow{6}{*}{64}
  & Shampoo$^{\sfrac{1}{4}}$   & 0.008 & 0.9875 & 0.9 & $10^{-30}$ \\
  &&& Shampoo$^{\sfrac{1}{2}}$   & 0.004 & 0.9875 & 0.95 & $10^{-15}$ \\
  &&& KL-Shampoo     & 0.004 & 0.9875 & 0.975 & $10^{-10}$ \\
  &&& Muon           & 0.004 & 0.9875 & -- & -- \\
  &&& AdamW          & 0.004 & 0.9875 & 0.99375 & \textcolor{BurntOrange}{$10^{-8}$} \\
  &&& Signum         & 0.00025 & 0.9875 & -- & -- \\
\midrule
\multirow{6}{*}{320M} & \multirow{6}{*}{$1\times$} & \multirow{6}{*}{256}
  & Shampoo$^{\sfrac{1}{4}}$   & 0.016 & 0.95 & 0.9 & $10^{-23}$ \\
  &&& Shampoo$^{\sfrac{1}{2}}$   & 0.016 & 0.95 & 0.8 & $10^{-15}$ \\
  &&& KL-Shampoo     & 0.016 & 0.95 & 0.8 & $10^{-10}$ \\
  &&& Muon           & 0.016 & 0.95 & -- & -- \\
  &&& AdamW          & 0.008 & 0.95 & 0.95 & \textcolor{BurntOrange}{$10^{-8}$} \\
  &&& Signum         & 0.001 & 0.9 & -- & -- \\
\midrule
\multirow{6}{*}{320M} & \multirow{6}{*}{$8\times$} & \multirow{6}{*}{256}
  & Shampoo$^{\sfrac{1}{4}}$   & 0.004 & 0.99375 & 0.95$^*$ & $10^{-27}$ \\
  &&& Shampoo$^{\sfrac{1}{2}}$   & 0.004 & 0.99375 & 0.95 & $10^{-13}$ \\
  &&& KL-Shampoo     & 0.004 & 0.99375 & 0.95 & $10^{-8}$ \\
  &&& Muon           & 0.008 & 0.9875 & -- & -- \\
  &&& AdamW          & 0.004 & 0.9875 & 0.95 & $10^{-7}$ \\
  &&& Signum         & 0.002 & 0.975 & -- & -- \\
\midrule
\multirow{6}{*}{1.5B} & \multirow{6}{*}{$1\times$} & \multirow{6}{*}{256}
  & Shampoo$^{\sfrac{1}{4}}$   & 0.004 & 0.9875 & 0.975 & $10^{-30}$ \\
  &&& Shampoo$^{\sfrac{1}{2}}$   & 0.002 & 0.9875 & 0.9 & $10^{-16}$ \\
  &&& KL-Shampoo     & 0.002 & 0.99375 & 0.95 & $10^{-8}$ \\
  &&& Muon           & 0.004 & 0.99375 & -- & -- \\
  &&& AdamW          & 0.004 & 0.99375 & 0.99375 & $10^{-12}$ \\
  &&& Signum         & 0.002 & 0.8 & -- & -- \\
\bottomrule
\end{tabular}
\end{table}

\begin{figure}[t]
    \centering
    \includegraphics[width=0.9\linewidth]{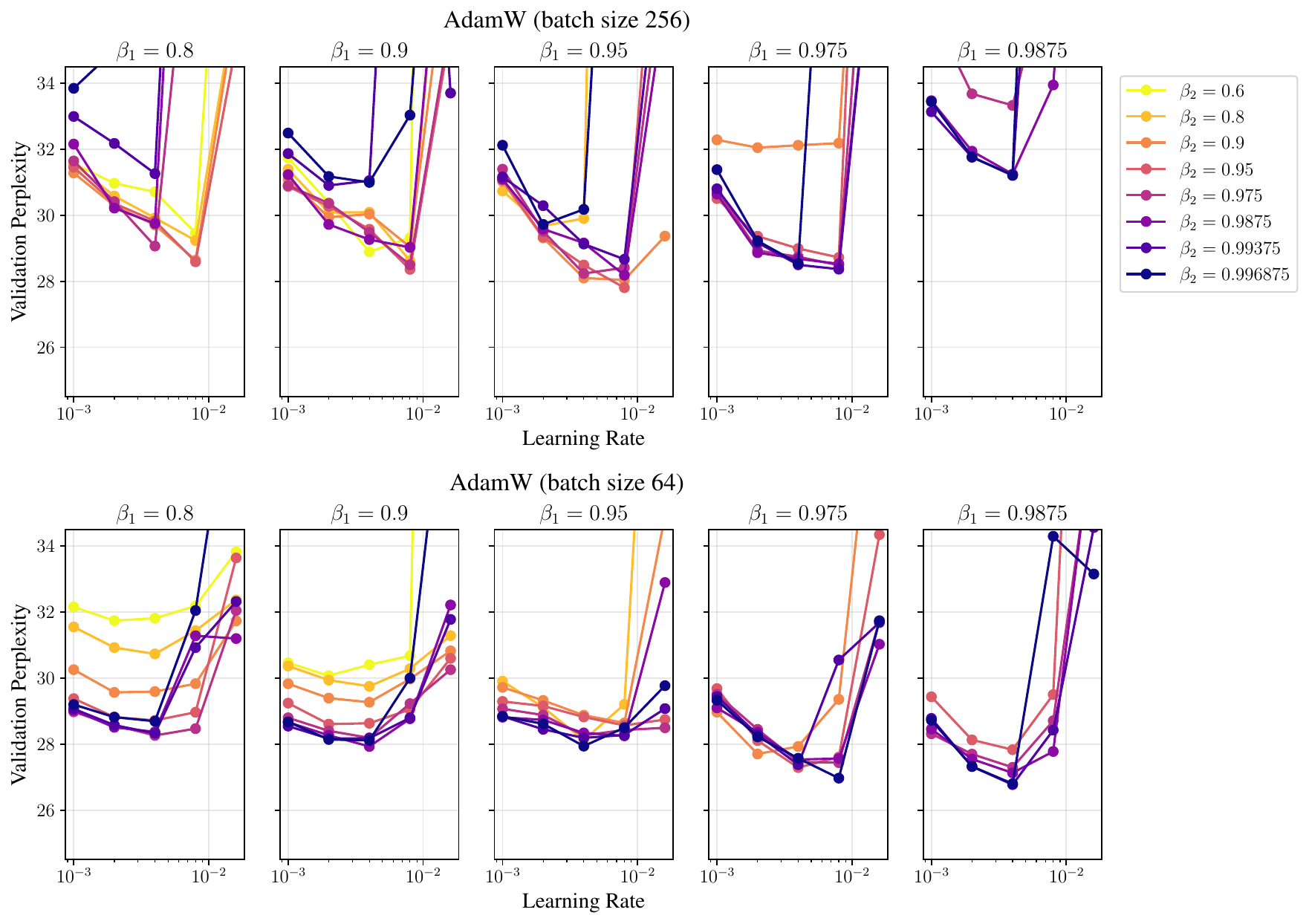}
    \caption{Final validation perplexity (single runs) across search space of learning rate and $\beta$s of AdamW in the 320M, $1\times$ token budget, batch size 256 and 64 settings. See \cref{tab:hparams-space} for search space ranges. Note that underperforming runs may not be shown.}
    \label{fig:hparam-plots-adamw}
\end{figure}

\begin{figure}[t]
    \centering
    \includegraphics[width=0.8\linewidth]{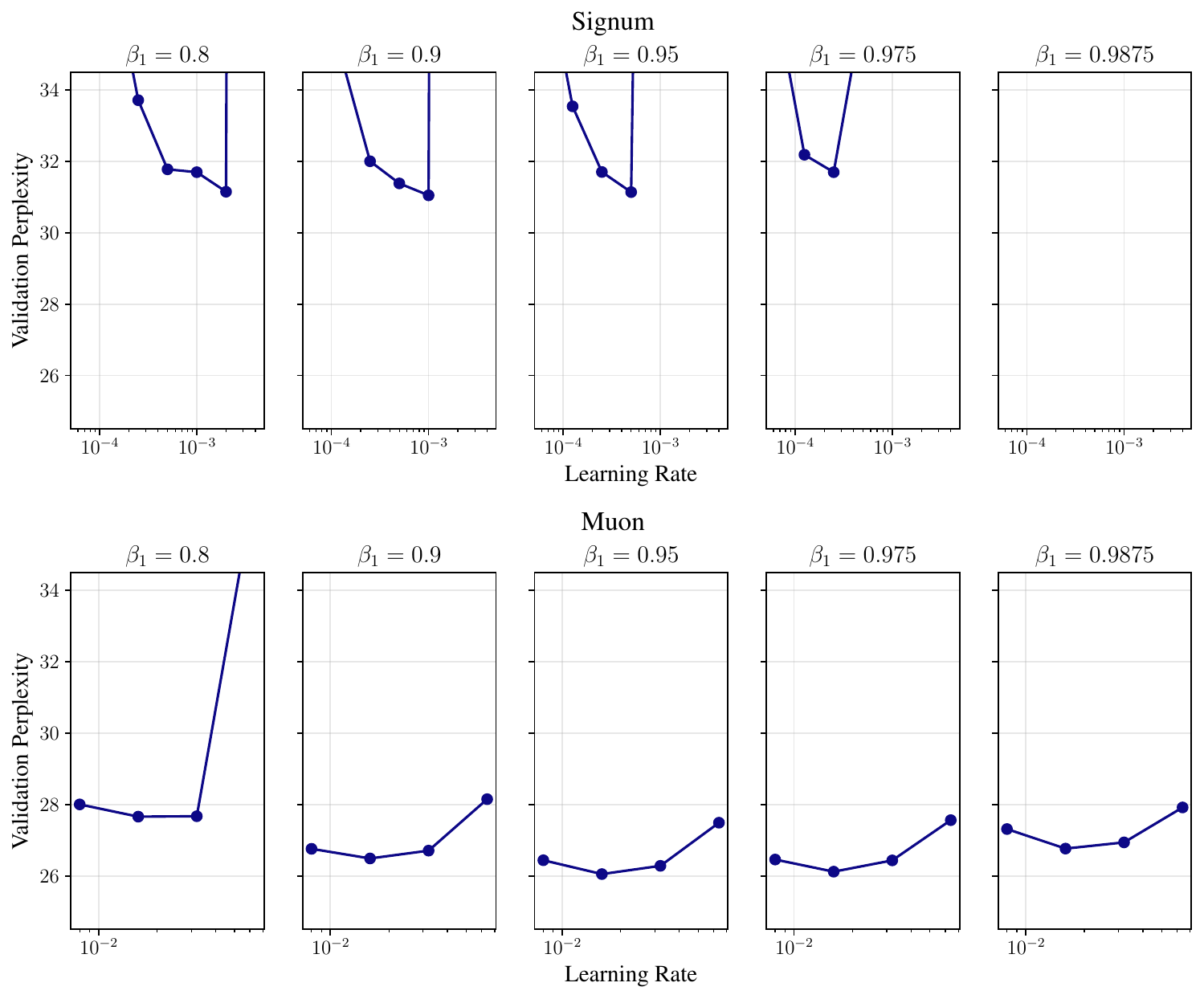}
    \caption{Final validation perplexity (single runs) across search space of learning rate and $\beta_1$ of Signum and Muon in the 320M, $1\times$ token budget, batch size 256 setting. See \cref{tab:hparams-space} for search space ranges. Note that underperforming runs may not be shown.}
    \label{fig:hparam-plots-other}
\end{figure}

\begin{figure}
    \centering
    \includegraphics[width=0.9\linewidth]{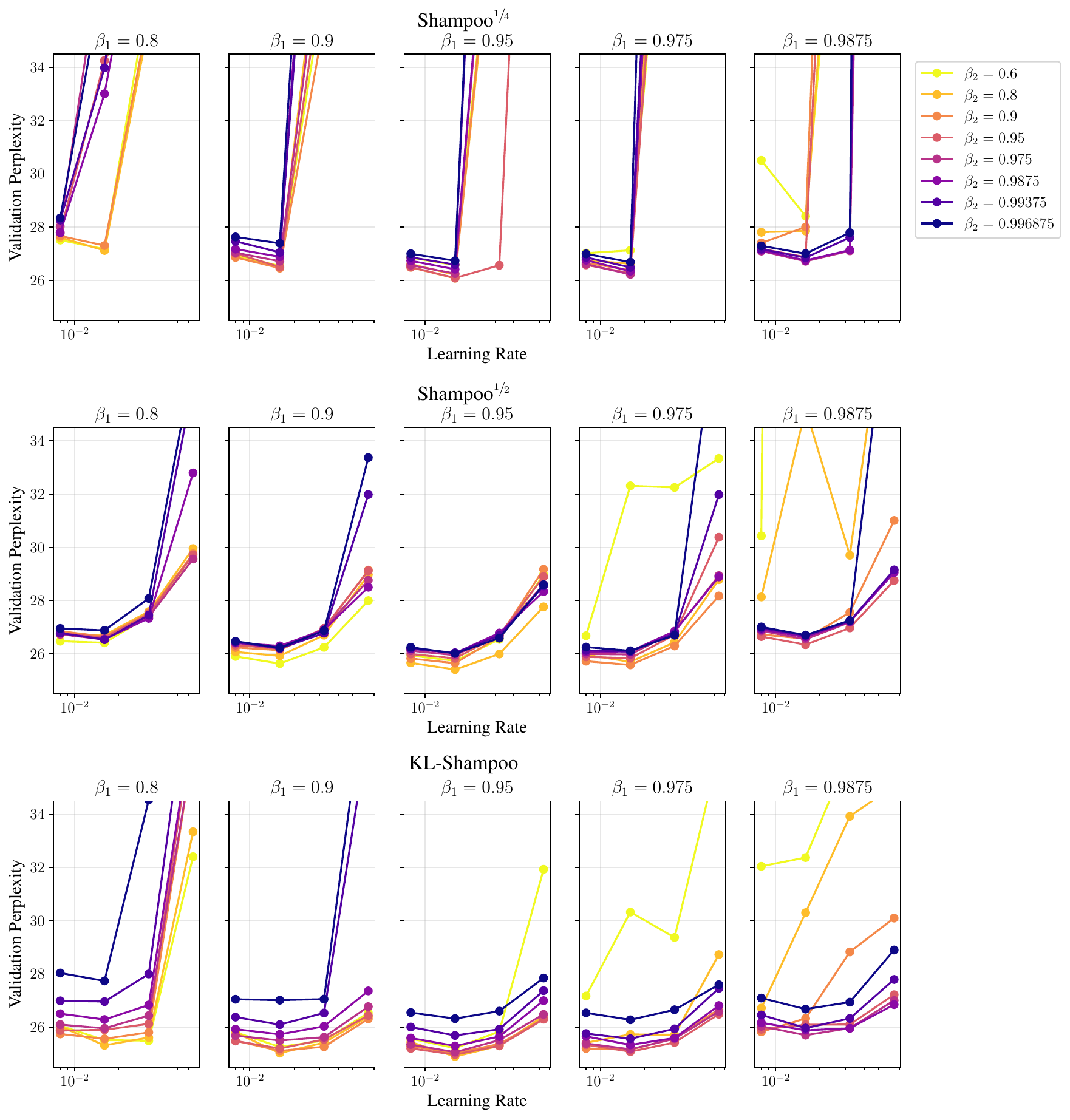}
    \caption{Final validation perplexity (single runs) across search space of learning rate and $\beta$s of Shampoo methods in the 320M, $1\times$ token budget, batch size 256 setting. See \cref{tab:hparams-space} for search space ranges. Note that underperforming runs may not be shown.}
    \label{fig:hparam-plots-shampoo}
\end{figure}

\begin{table}[h!]
    \caption{Best hyperparameters for all experiments on language models outside of \cref{tab:core-methods}. See the reference column for the corresponding table with experimental results.}
    \label{tab:combined-hparams}
    \centering
    \begin{tabular}{ll cccc}
\toprule
Reference & Method & $\alpha$ & $\beta_1$ & $\beta_2$ & $\epsilon$ \\
\midrule
\multirow{2}{*}{\cref{fig:summary}} & AdamW (reduced budget) & 0.004 & 0.975 & 0.99375 & $10^{-8}$ \\
 & KL-Shampoo (reduced budget) & 0.016 & 0.95 & 0.95 & $10^{-10}$ \\
\midrule
\multirow{2}{*}{\cref{tab:eshampoo}} & EShampoo (no grafting) & 0.016 & 0.975 & 0.99375 & $10^{-6}$ \\
 & EShampoo (grafting) & 0.016 & 0.95 & 0.99375 & $10^{-30}$ \\
\midrule
\multirow{2}{*}{\cref{tab:muon-scalings}} & Muon (classic) & 0.016 & 0.975 & -- & -- \\
 & Muon (moonlight) & 0.016 & 0.95 & -- & -- \\
\midrule
\multirow{2}{*}{\cref{tab:shampoo-otherparams}} & Shampoo$^{\sfrac{1}{2}}$ (1D) & 0.008 & 0.95 & 0.9 & $10^{-15}$ \\
 & Shampoo$^{\sfrac{1}{2}}$ (embed.) & 0.008 & 0.95 & 0.8 & $10^{-8}$ \\
\midrule
\multirow{2}{*}{\cref{tab:one-sided-shampoo}} & (KL-)Shampoo $\mL_t$ Only & 0.016 & 0.95 & 0.95 & $10^{-16}$ \\
 & (KL-)Shampoo $\mR_t$ Only & 0.016 & 0.95 & 0.8 & $10^{-16}$ \\
\midrule
\multirow{2}{*}{\cref{tab:kl-shampoo-variants}} & LaProp KL-Shampoo & 0.016 & 0.95 & 0.9 & $10^{-9}$ \\
 & BCOS-m KL-Shampoo & 0.016 & 0.975 & 0.8 & $10^{-11}$ \\
\bottomrule
\end{tabular}
\end{table}

We note some observations on the general hyperparameter trends across our experiments. 

\paragraph{Correlation between $\beta_1$ and $\beta_2$. } We find that for all relevant methods (AdamW and the Shampoo variants), the optimal $\beta_1$ and $\beta_2$ values are strongly positively correlated.

This is notably consistent with the AdamW experiments by \citet{orvieto2025search}. However, we do not observe that smaller differences between $\beta_1$ and $\beta_2$ consistently correlate with lower loss across all settings. For instance, in the 320M $1\times$ token budget settings, we find that loss is generally lowest when $\beta$s are closer for batch size 256, but we do not find this to be the case for batch size 64 (see \cref{fig:hparam-plots-adamw}).

\paragraph{Influence of grafting. } We hypothesize that the use of AdamW as the grafted optimizer may have affected the optimal hyperparameters of the matrix optimizers; however, we have not experimented with removing grafting from any of the methods except Muon.
One observation likely related to this influence is that all methods using AdamW grafting have the same or very similar optimal $\beta_1$ as AdamW within the same settings.
We also observe that for the Shampoo variants, the optimal $\beta_2$ values tend to be lower than the optimal $\beta_1$ values (\cref{fig:hparam-plots-shampoo}). This could be due to grafting, which removes the influence of Shampoo's $\beta_2$ on the update magnitude.

\paragraph{Learning rates. } With the exception of the 320M, 1$\times$ token budget, batch size 256 setting, we do not observe that the matrix-based optimizers (Shampoo variants and Muon) benefit from larger learning rates compared to element-wise methods (AdamW and Signum). This contrasts with the findings in \citet{semenov2025benchmarking}. However, it is possible that our learning rate sweeps were not granular enough to capture this trend.
Additionally, we observe that Signum uses much smaller learning rate than the other methods, which is consistent with the experiments by \citet{orvieto2025search}.

\paragraph{$\epsilon$ in Shampoo. }
\label{sec:epsilon-role}

We observe that optimal $\epsilon$ values for the different Shampoo methods are quite different. For example, in the 320M model setting with $1\times$ token budget and batch size 256, the best values of $\epsilon$ for Shampoo$^{\sfrac{1}{4}}$, Shampoo$^{\sfrac{1}{2}}$, and KL-Shampoo are $10^{-23}$, $10^{-15}$, and $10^{-10}$, respectively (\Cref{tab:hparams-core}).
Notably, the optimal $\epsilon$ for Shampoo$^{\sfrac{1}{4}}$ is significantly smaller than for Shampoo$^{\sfrac{1}{2}}$. 
We attribute this discrepancy to the order of operations: $\epsilon$ is added to the factor matrices \emph{before} computing the root, unlike in Adam. 
If we compare the effective damping by examining the root of $\epsilon$ itself, the values align more closely, \ie, $\epsilon^{1/4} \approx 1.78 \cdot 10^{-6}$ for Shampoo$^{\sfrac{1}{4}}$ and $\epsilon^{1/2} \approx 3.16 \cdot 10^{-8}$ for Shampoo$^{\sfrac{1}{2}}$.
The choice of $\epsilon$ in KL-Shampoo is more subtle, as it influences both the preconditioning and the factor matrices themselves; see \Cref{sec:kl-shampoo-implementation}.

\paragraph{Effect of batch size, token budget, and model size. } Since we only test two different settings for each of batch size, token budget, and model size, we cannot draw any definitive conclusions about their effects on various hyperparameters. However, we make some initial observations on the patterns that exist in our data.

We observe that increasing the batch size (while keeping all other settings constant) increases the optimal learning rate and decreases the optimal $\beta$s, which is consistent with the findings in \citet{marek2025smallbatch}.
We also observe that increasing the token budget or model size (keeping all other settings constant) causes the optimal learning rate to decrease. 
Using the maximal-update parameterization ($\mu$P) can yield better hyperparameter transfer across model sizes \citep{yang2024spectral}.

\subsubsection{Comparison to Prior Work on Language Model Optimizer Benchmarking}

Recent benchmarks by \citet{semenov2025benchmarking} and \citet{wen2025fantastic} evaluate 10 -- 11 optimizers across dense and Mixture-of-Experts language models, utilizing different sequence lengths, tokenizers, and datasets compared to our study. 
\citet{semenov2025benchmarking} find that while matrix-based optimizers generally outperform AdamW, variance reduction methods like Mars \citep{yuan2024mars} and AdEMAMix \citep{pagliardini2024ademamix} are most effective in small-batch regimes. 
Conversely, \citet{wen2025fantastic} show matrix-based optimizers outperform element-wise optimizers (including variance-reduced methods) in large-batch regimes, attributing the discrepancy to batch size effects, though they note diminishing returns at scale and caution that rankings often flip during the learning rate decay phase. 

Evidence comparing SOAP and Muon is mixed: \citet{wen2025fantastic} prefer Muon for small token budgets and SOAP for larger ones, while \citet{semenov2025benchmarking} find SOAP superior up to 210M parameters, but outperformed by AdamW and D-Muon (a Muon variant with corrected weight decay) at larger scale.
These discrepancies likely stem from numeric confounders; neither benchmark updates SOAP's eigenbasis at every iteration, carefully tunes $\epsilon$, or removes blocking.
Our experiments explicitly control for these confounders.

Closest to our work, \citet{frans2025really} observe that SOAP generally outperforms Muon, suggesting that its gains do not stem solely from orthogonalization.
Our results broadly support this conclusion.

\subsection{Image Classification}
\label{sec:image-experiments}

\subsubsection{Experimental Setup}

All models are trained on the Imagewoof dataset with randomized cropping and horizontal flips as data augmentation and with a cross entropy loss function, using $1\times$ NVIDIA A100 80GB GPU per run.\footnote{The Imagewoof dataset is available at \url{https://github.com/fastai/imagenette}.}
For the vision transformer (ViT) model, we use SimpleViT \citep{beyer2022better} with patch size $16$, $6$ heads, a depth of $12$ layers, an MLP dimension of $192$, dimension of $48$, gradient clipping with threshold $1$, and weight decay of $10^{-4}$.
For the ConvNeXt V2 architecture \citep{woo2023convnextv2} we use weight decay of $0.05$ and drop paths with rate $0.1$.

\begin{table}[t]
    \caption{Hyperparameters for \Cref{fig:full-batch}. Both matrix optimizers are using grafting with the optimal hyperparameters of AdamW/RMSProp.}
    \label{tab:fig3-optimal-hparams}
    \centering
    \begin{tabular}{l cccc}
\toprule
Method & $\alpha$ & $\beta_1$ & $\beta_2$ & $\epsilon$ \\
\midrule
SignGD & $0.001$ & $0$ & - & - \\
RMSProp & $0.001$ & $0$ & $0.99921875$ & $10^{-10}$ \\
SpectralGD & $0.003$ & $0$ & - & - \\
Shampoo$^{\sfrac{1}{2}}$ & $0.003$ & $0$ & $0.8$ & $10^{-12}$ \\
\bottomrule
\end{tabular}
\end{table}

\paragraph{\Cref{fig:reshape}.}
All models are trained for $100$ epochs, using a learning rate schedule consisting of a linear warmup for 353 steps followed by cosine decay. We use a batch size of $128$ and the default settings for $\beta_1 = 0.9$ and $\beta_2 = 0.999$.
The optimal learning rate for AdamW is $0.001$ and we use grafting with the same hyperparameters for all Shampoo runs.

\paragraph{\Cref{fig:full-batch}.}
To be able to train in a full-batch setting, we consider a subsampled dataset, using 5\% of examples of each class.
All models are trained for $1000$ steps/epochs, using linear warmup for 50 steps followed by cosine decay.
For SignGD and SpectralGD, we sweep the learning rate.
For RMSProp and Shampoo$^{\sfrac{1}{2}}$ without momentum ($\beta_1=0$), we sweep the learning rate and $\beta_2$ and use grafting from RMSProp (using its $\epsilon$ and optimal $\beta_2$); see \Cref{tab:fig3-optimal-hparams} for the optimal hyperparameters.

\section{Proofs}
\label{sec:proofs-appendix}

\optimaladaptationgeneral*
\begin{proof}
    Appendix B.3 in \citet{balles2017dissecting}.
\end{proof}

\optimaladaptationsign*
\begin{proof}
    Assume that $\mathbb{E}[g_i^2] > 0$ for all $i$. Then we have that
    \begin{equation}
    \begin{split}
        \E \left[ || \bm{\gamma} \odot \vg - \mathrm{sign}\left( \nabla \gL \right) ||_2^2 \right] &= \sum_{i=1}^d \E \left[ \left( \gamma_i g_i - \mathrm{sign}\left( \nabla \gL_i \right) \right)^2 \right] \\
        &= \sum_{i=1}^d \gamma_i^2 \E\left[ g_i^2 \right] - 2 \gamma_i (\nabla \gL_i) \cdot \mathrm{sign}(\nabla \gL_i) + \mathrm{sign}(\nabla \gL_i)^2.
    \end{split}
    \end{equation}
    Note that the function is separable with respect to $\gamma_i$ for $i = 1, ..., d$. 
    By computing the derivative \wrt $\gamma_i$ and setting it to zero, we get the solution
    \begin{equation}
        \gamma_i = \frac{(\nabla \gL_i) \cdot \mathrm{sign}\left( \nabla \gL_i \right) }{\E \left[ g_i^2 \right]} = \frac{|\nabla \gL_i|}{\E \left[ g_i^2 \right]}.
    \end{equation}
\end{proof}

\optimaladaptationspectral*
\begin{proof}
With $\mathrm{rank}(\nabla \gL) = k$, we can rewrite the objective as
\begin{equation}
\label{eq:A-spectral-objective}
\begin{split}
    \E \left[ || \mA \mG - \mU \mV^\transpose ||_F^2 \right] &= \E \left[ || \mA \mG - \left( \nabla \gL \nabla \gL^\transpose \right)^{\frac{\dagger}{2}} \nabla \gL ||_F^2 \right] \\
    &= \E \left[ \Tr\left( \left( \mA \mG - \left( \nabla \gL \nabla \gL^\transpose \right)^{\frac{\dagger}{2}} \nabla \gL \right)^\transpose \left( \mA \mG - \left( \nabla \gL \nabla \gL^\transpose \right)^{\frac{\dagger}{2}} \nabla \gL \right) \right) \right] \\
    &= \E \left[ \Tr\left( \left( \mG^\transpose \mA^\transpose - \nabla \gL^\transpose \left( \nabla \gL \nabla \gL^\transpose \right)^{\frac{\dagger}{2}} \right) \left( \mA \mG - \left( \nabla \gL \nabla \gL^\transpose \right)^{\frac{\dagger}{2}} \nabla \gL \right) \right) \right] \\
    &= \E \left[ \Tr\left( \mG^\transpose \mA^\transpose \mA \mG \right) - 2 \Tr\left( \mG^\transpose \mA^\transpose \left( \nabla \gL \nabla \gL^\transpose \right)^{\frac{\dagger}{2}} \nabla \gL \right) + \Tr\left( \nabla \gL^\transpose \left( \nabla \gL \nabla \gL^\transpose \right)^{\dagger} \nabla \gL \right) \right] \\
    &= \Tr\left( \mA \E \left[ \mG \mG^\transpose \right] \mA^\transpose \right)  - 2 \Tr\left( \nabla \gL^\transpose \mA^\transpose \left( \nabla \gL \nabla \gL^\transpose \right)^{\frac{\dagger}{2}} \nabla \gL \right) + k \\
    &= || \mA \nabla \gL - \left( \nabla \gL \nabla \gL^\transpose \right)^{\frac{\dagger}{2}} \nabla \gL ||_F^2 + \Tr\left( \mA \mathrm{Cov}_\mathrm{row}\left( \mG \right) \mA^\transpose \right),
\end{split}
\end{equation}
where $^\dagger$ denotes the pseudoinverse.
Computing the derivative and setting it to zero,
\begin{equation}
    \frac{\partial}{\partial \mA} \left( \E \left[ || \mA \mG - \mU \mV^\transpose ||_F^2 \right] \right) = 2 \mA \E \left[  \mG \mG^\transpose \right] - 2 \left( \nabla \gL \nabla \gL^\transpose \right)^{\frac{\dagger}{2}} \nabla \gL \nabla \gL^\transpose = \mathbf{0},
\end{equation}
and assuming $\E[ \mG \mG^\transpose ]$ is full rank, yields the optimal solution
\begin{equation}
\label{eq:optimal-A-spectral}
\begin{split}
    \mA &= \left( \nabla \gL \nabla \gL^\transpose \right)^{ \frac{1}{2}} \E\left[ \mG \mG^\transpose \right]^{-1}. \\
\end{split}
\end{equation}
\end{proof}

\matrixwhitening*

\whitening*
\begin{proof}
    We have that
    \begin{equation}
    \begin{split}
        \mathrm{Cov}_\mathrm{row}\left( \mA \mG \mB \right) = \mI, \hspace{1em} & \hspace{1em} \mathrm{Cov}_\mathrm{col}\left( \mA \mG \mB \right) = \mI \\
        \equiv \mA \mathrm{Cov}_\mathrm{row}\left( \mG \mB \right) \mA^\transpose = \mI, \hspace{1em} & \hspace{1em} \mB^\transpose \mathrm{Cov}_\mathrm{col}\left( \mA \mG \right) \mB = \mI \\
        \equiv \mA = \mathrm{Cov}_\mathrm{row}\left( \mG \mB \right)^{-\frac{1}{2}}, \hspace{1em} & \hspace{1em} \mB = \mathrm{Cov}_\mathrm{col}\left( \mA \mG \right)^{-\frac{1}{2}},
    \end{split}
    \end{equation}
    assuming the covariances are full rank.
\end{proof}

\taoe*

\idealklshampoo*
\begin{proof}
    We have the optimality condition
\begin{align*}
    \mathrm{EMA}_{t=1}^{T} \left({\E[\mZ_t \mZ_t^\transpose}] \right) = \mathbf{I}_m, &\hspace{1em} \mathrm{EMA}_{t=1}^{T} \left( {\E[\mZ_t^\transpose \mZ_t}] \right) = \mathbf{I}_n \\
    \equiv \hspace{1em} (1-\beta_2) \sum_{t=1}^T \beta_2^{T-t} \E[\mA \mG_t \mB \mB^\transpose \mG_t^\transpose \mA^\transpose] = \mathbf{I}_m, &\hspace{1em} (1-\beta_2) \sum_{t=1}^T \beta_2^{T-t} \E[\mB^\transpose \mG_t^\transpose \mA^\transpose \mA \mG_t \mB] = \mathbf{I}_n \\
    \equiv \hspace{1em} \mA \left( (1-\beta_2) \sum_{t=1}^T \beta_2^{T-t} \E[ \mG_t \mB \mB^\transpose \mG_t^\transpose ] \right) \mA^\transpose = \mathbf{I}_m, &\hspace{1em} \mB^\transpose \left( (1-\beta_2) \sum_{t=1}^T \beta_2^{T-t} \E[ \mG_t^\transpose \mA^\transpose \mA \mG_t ] \right) \mB = \mathbf{I}_n \\
    \equiv \hspace{1em} \mathrm{EMA}_{t=1}^{T} \left( \E[ \mG_t \mB \mB^\transpose \mG_t^\transpose ] \right) = \mA^{-2}, &\hspace{1em} \mathrm{EMA}_{t=1}^{T} \left( \E[ \mG_t^\transpose \mA^\transpose \mA \mG_t ] \right) = \mB^{-2} \\
    \equiv \hspace{1em} \mA = \mathrm{EMA}_{t=1}^{T} \left( \E[ \mG_t \mB^{2} \mG_t^\transpose ] \right)^{-\frac{1}{2}}, &\hspace{1em} \mB = \mathrm{EMA}_{t=1}^{T} \left( \E[ \mG_t^\transpose \mA^{2} \mG_t ] \right)^{-\frac{1}{2}}, \\
\end{align*}
where we assume that the EMAs are full rank in the last step.
\end{proof}

\instantklshampoo*

\begin{proof}
    With $\sigma_i$ being the $i$th singular value on the diagonal of $\mSigma \in \R^{k \times k}$, we have
\begin{equation}
\begin{aligned}
    \mL_t &= \beta_2 \mL_{t-1} + (1 - \beta_2) \mG \mR_{t-1}^{-1} \mG^\transpose \\
    &= \beta_2 \mL_{t-1} + (1 - \beta_2) \mU \left( \mSigma \mV^\transpose \mR_{t-1}^{-1} \mV \mSigma^\transpose \right) \mU^\transpose \\
    \bar{\mL}_t &:= \mU^\transpose \mL_t \mU \\
    &= \beta_2 \bar{\mL}_{t-1} + (1 - \beta_2) \mSigma \bar{\mR}_{t-1}^{-1} \mSigma^\transpose \\
    \\
    \mR_t &= \beta_2 \mR_{t-1} + (1 - \beta_2) \mG^\transpose \mL_{t-1}^{-1} \mG \\
    &= \beta_2 \mR_{t-1} + (1 - \beta_2) \mV \left( \mSigma^\transpose \mU^\transpose \mL_{t-1}^{-1} \mU \mSigma \right) \mV^\transpose \\
    \bar{\mR}_t &:= \mV^\transpose \mR_t \mV \\
    &= \beta_2 \bar{\mR}_{t-1} + (1 - \beta_2) \mSigma^\transpose \bar{\mL}_{t-1}^{-1} \mSigma. \\
\end{aligned}
\end{equation}
Note that if we initialize $\mL_0 = \mR_0 = c\mathbf{I}$ for any $c \in \R^+$, the matrices $\bar{\mL}_t$ and $\bar{\mR}_t$ will remain diagonal matrices.
Hence, we can simplify the update and express it in terms of the diagonal entries of $\bar{\mL}_t$, $l_{t, i}$, and of $\bar{\mR}_t$, $r_{t, i}$, for $i = 1, \dots, k$, where $k$ is the rank of the matrix $\mG$.
We drop the subscript $i$, since the update evolves in the same way for each $i$.
We have
\begin{equation}
\begin{aligned}
    l_t &= \beta_2 l_{t-1} + (1 - \beta_2) \frac{\sigma^2}{r_{t-1}} \\
    r_t &= \beta_2 r_{t-1} + (1 - \beta_2) \frac{\sigma^2}{l_{t-1}}. \\
\end{aligned} 
\end{equation}
Since we initialize $l_t=r_t$, the sequences will evolve identically and we will consider them jointly using $x_t=l_t=r_t$:
\begin{equation}
    x_t = \beta_2 x_{t-1} + (1 - \beta_2) \frac{\sigma^2}{x_{t-1}}
\end{equation}
We want to show $x_t \rightarrow \sigma$ for $t \rightarrow \infty$.
First, note that $\sigma$ is a fixed point for the function $f(x) = \beta_2 x + (1 - \beta_2) \frac{\sigma^2}{x}$, \ie $f(\sigma) = \sigma$.
We have $f'(x) = \beta_2 - (1 - \beta_2) \frac{\sigma^2}{x^2}$ and $f'(\sigma) = \beta_2 - (1 - \beta_2) = 2 \beta_2 - 1$.
For the iteration to converge to the fixed point $\sigma$, we need $|f'(\sigma)| = |2\beta_2 - 1| < 1$, so we have $x_t \rightarrow \sigma$ for $\beta_2 \in (0, 1)$.

Note that if we choose $\beta_2 = 0.5$, the iteration becomes $x_t = \frac{1}{2} (x_{t-1} + \frac{\sigma^2}{x_{t-1}})$. We can define a function $g(x) = x^2 - \sigma^2$ with derivative $g'(x) = 2x$ and use Newton's method to find the root of $g(x)$, \ie 
\begin{equation}
\begin{split}
    x_{t} &= x_{t-1} - \frac{g(x_{t-1})}{g'(x_{t-1})} \\
    &= x_{t-1} - \frac{x_{t-1}^2 - \sigma^2}{2x_{t-1}} \\
    &= x_{t-1} - \frac{x_{t-1} - \frac{\sigma^2}{x_{t-1}}}{2} \\
    &= \frac{1}{2} \left( x_{t-1} + \frac{\sigma^2}{x_{t-1}} \right),
\end{split}
\end{equation}
which exactly recovers our iteration.
In this case, we have quadratic convergence.

Now we can transform this result back to the original space
\begin{equation}
\begin{aligned}
    \mL_\infty : = \lim_{t \rightarrow \infty} \mL_t &= \lim_{t \rightarrow \infty} \mU \bar{\mL}_t \mU^\transpose = \mU \mSigma \mU^\transpose = \left( \mG \mG^\transpose \right)^\frac{1}{2}, \\
    \mR_\infty : = \lim_{t \rightarrow \infty} \mR_t &= \lim_{t \rightarrow \infty} \mV \bar{\mR}_t \mV^\transpose = \mV \mSigma \mV^\transpose = \left( \mG^\transpose \mG \right)^\frac{1}{2}. \\
\end{aligned}
\end{equation}
When we now precondition $\mG$ with these matrices, using $p=1/2$, we have
\begin{equation}
\label{eq:alt-ortho}
\begin{split}
    \mL_\infty^{\frac{\dagger}{2}} \mG \mR_\infty^{\frac{\dagger}{2}} &= \left(\left( \mG \mG^\transpose \right)^\frac{1}{2} \right)^{\frac{\dagger}{2}} \mG \left(\left( \mG^\transpose \mG \right)^\frac{1}{2}\right)^{\frac{\dagger}{2}} \\
    &= \left( \mG \mG^\transpose \right)^{\frac{\dagger}{4}} \mG \left( \mG^\transpose \mG \right)^{\frac{\dagger}{4}} = \mU \mV^\transpose.
\end{split}
\end{equation}
\end{proof}

\section{Additional Discussion}

\subsection{Variance Adaptation}
\label{sec:notes-variance-adaptation}

We can generalize \Cref{prop:one-sided-optimal-adaptation-spectral}:

\paragraph{$\mB = \mathbf{I}$.}
See \Cref{prop:one-sided-optimal-adaptation-spectral} and \Cref{sec:proofs-appendix}.

\paragraph{$\mA = \mathbf{I}$.}

Analogously, the solution is
\begin{equation}
    \mB = \E\left[ \mG^\transpose \mG \right]^{-1} \left( \nabla \gL^\transpose \nabla \gL \right)^{ \frac{1}{2}}.
\end{equation}
With this we can rewrite the corrected stochastic direction $\mG = \hat{\mU} \hat{\mSigma} \hat{\mV}^\transpose$ as
\begin{equation}
     \mG \mB = \hat{\mU} \hat{\mV}^\transpose \left( \mG^\transpose \mG \right)^{\frac{1}{2}} \E\left[ \mG^\transpose \mG \right]^{-1} \left( \nabla \gL^\transpose \nabla \gL \right)^{ \frac{1}{2}}.
\end{equation}

In the deterministic setting, we have $\mA = \left( \nabla \gL \nabla \gL^\transpose \right)^{\frac{\dagger}{2}}$ or $\mB = \left( \nabla \gL^\transpose \nabla \gL \right)^{\frac{\dagger}{2}}$, which means $\mA$ or $\mB$ just orthogonalizes $\nabla \gL$, \ie the solution is exact.

\paragraph{General case.}
We have the following optimality condition:
\begin{equation}
\begin{split}
    \mA &= \left( \left( \nabla \gL \nabla \gL^\transpose \right)^{\frac{\dagger}{2}} \nabla \gL \mB^\transpose \nabla \gL^\transpose \right) \E\left[ \mG \mB \mB^\transpose \mG^\transpose \right]^{-1}, \\
    \mB &= \E\left[ \mG^\transpose \mA^\transpose \mA \mG \right]^{-1} \left( \nabla \gL^\transpose \mA^\transpose \nabla \gL \left( \nabla \gL^\transpose \nabla \gL \right)^{\frac{\dagger}{2}} \right).
\end{split}
\end{equation}

\paragraph{Dual use of $\mM$.}
Let $\mM = \bar{\mU} \bar{\mSigma} \bar{\mV}^\transpose$ be an estimate of $\nabla \gL$ based on multiple samples of $\mG$, \eg an EMA.
Analogously to the Adam case in \Cref{eq:optimal-variance-approx}, we will now dual-use this estimate by 1) replacing the stochastic gradient $\mG$ that we want to adapt with this, and 2) approximate the true gradient $\nabla \gL$ with it as part of the optimal solution (we denote the solution with this approximation as $\bar{\mA}, \bar{\mB}$).
We consider the case in \Cref{prop:one-sided-optimal-adaptation-spectral} for simplicity, but the other cases follow similarly.

We first apply $\bar{\mA}$ to $\mM$:
\begin{equation}
\begin{split}
    \bar{\mA} \mM &= \bar{\mA} \left( \mM \mM^\transpose \right)^{\frac{1}{2}} \bar{\mU} \bar{\mV}^\transpose \\
    &= \left( \mM \mM^\transpose \right)^{ \frac{1}{2}} \E\left[ \mG \mG^\transpose \right]^{-1} \left( \mM \mM^\transpose \right)^{\frac{1}{2}} \bar{\mU} \bar{\mV}^\transpose \\
    &= \mathcolor{Maroon}{\bar{\mU}} \mathcolor{MidnightBlue}{\left( \bar{\mSigma} \left( \bar{\mU}^\transpose \E\left[ \mG \mG^\transpose \right]^{-1} \bar{\mU} \right) \bar{\mSigma} \right)} \mathcolor{Maroon}{\bar{\mV}^\transpose}. \\
\end{split}
\end{equation}
In the case where $\bar{\mU}^\transpose \E\left[ \mG \mG^\transpose \right]^{-1} \bar{\mU} =: \mD$ is diagonal, we have
\begin{equation}
    \mathcolor{Maroon}{\bar{\mU}} \mathcolor{MidnightBlue}{\left( \bar{\mSigma}^2 \mD \right)} \mathcolor{Maroon}{\bar{\mV}^\transpose},
\end{equation}
where the whole matrix is the middle, the singular values of the update, is diagonal and the left and right singular vectors are shared with $\mM$.
The left-sided (KL-)Shampoo update without time-averaging and with an expectation instead of a single sample can be written as
\begin{equation}
\begin{split}
    \mL^{-\frac{1}{2}} \mM &= \mathcolor{Maroon}{\bar{\mU}} \mathcolor{MidnightBlue}{\left( \left( \bar{\mU}^\transpose \E\left[ \mG \mG^\transpose \right]^{-\frac{1}{2}} \bar{\mU} \right) \bar{\mSigma} \right)} \mathcolor{Maroon}{\bar{\mV}^\transpose},
\end{split}
\end{equation}
assuming $\mM$ has full row rank, \ie $\bar{\mU} \bar{\mU}^\transpose = \mathbf{I}$.
To highlight the similarity to the Adam case in \Cref{eq:adam-in-terms-of-optimal}, we again assume diagonality and plug in $\mD$:
\begin{equation}
    \mathcolor{Maroon}{\bar{\mU}} \mathcolor{MidnightBlue}{\left( \bar{\mSigma} \mD^{-\frac{1}{2}} \right) \mathcolor{Maroon}{\bar{\mV}^\transpose} = \mathcolor{Maroon}{\bar{\mU}} \left( \bar{\mSigma}^2 \mD \right)}^{\textcolor{orange}{\frac{1}{2}}} \mathcolor{Maroon}{\bar{\mV}^\transpose},
\end{equation}
where the only difference to the optimal solution is the square root, just like in the Adam case in \Cref{eq:adam-in-terms-of-optimal}.

Finally, this suggests an alternative decomposition of the Shampoo update in \Cref{eq:shampoo-decomposition},
\begin{equation}
    \mL_t^{-p} \mM_t \mR_t^{-p} = \mathcolor{Maroon}{\bar{\mU}} \mathcolor{MidnightBlue}{\left( \bar{\mU}^\transpose \mL_t^{-p} \bar{\mU} \right) \bar{\mSigma} \left( \bar{\mV}^\transpose \mR_t^{-p} \bar{\mV} \right)} \mathcolor{Maroon}{\bar{\mV}^\transpose},
\end{equation}
assuming $\mM_t$ is full rank.

\subsection{Whitening}

\begin{definition}[Whitening of a random vector]
\label{def:whitening-vector}
    Let $\hat{\vg} \in \R^{m}$ be a random vector with mean $\E\left[ \hat{\vg}\right] = \vg$ and covariance matrix $\mathrm{Cov}\left( \hat{\vg} \right) = \E\left[ \hat{\vg} \hat{\vg}^\transpose \right] - \vg \vg^\transpose$.
    We call the multiplication of a symmetric positive definite matrix $\mA \in \R^{m \times m}$ with $\hat{\vg}$ \emph{whitening} iff $\mathrm{Cov}\left( \mA \hat{\vg} \right) = \mI$.
\end{definition}
\begin{corollary}[Whitening matrix]
    The matrix that whitens a random vector according to \Cref{def:whitening-vector} is
    \begin{equation}
        \mA_\mathrm{whitening} = \mathrm{Cov}\left( \hat{\vg} \right)^{-\frac{1}{2}}.
    \end{equation}
\end{corollary}
\begin{proof}
    By \Cref{def:whitening-vector}, we have
    \begin{equation}
    \begin{split}
        &\mathrm{Cov}\left( \mA \hat{\vg} \right) = \mI \\
        \equiv &\mA \mathrm{Cov}\left( \hat{\vg} \right) \mA^\transpose = \mI \\
        \equiv &\mA = \mathrm{Cov}\left( \hat{\vg} \right)^{-\frac{1}{2}},
    \end{split} 
    \end{equation}
    assuming the covariance is full rank. Note that the solution is only unique because we constrained $\mA$ to be symmetric \citep{kessy2018optimal}.
\end{proof}

\subsection{What Should We Precondition?}
\label{sec:what-to-precondition}

No matter what interpretation of Shampoo we consider, we have to decide which matrix to precondition, specifically the gradient $\mG_t$ or EMA $\mM_t$.
In the standard Shampoo update, we precondition $\mM_t$ with a statistic of the gradient $\mG_t$, \ie according to all interpretations based on adaptation to stochasticity, there is a mismatch.\footnote{The same applies to Adam.}
To resolve this potential mismatch, we can either consider LaProp KL-Shampoo (preconditioning $\mG_t$; \cf \Cref{sec:laprop-shampoo}) or BCOS-m KL-Shampoo (preconditioning with statistic of $\mM_t$; \cf \Cref{sec:bcosm-shampoo}).
In practice, we find that KL-Shampoo matches LaProp KL-Shampoo and outperforms BCOS-m KL-Shampoo, see \Cref{tab:kl-shampoo-variants}.
\begin{table}[ht]
\caption{KL-Shampoo variants; same setting as \Cref{tab:core-methods}, 320M parameters, $1\times$ token budget, and batch size $256$.}
\label{tab:kl-shampoo-variants}
\centering
\begin{tabular}{Glll}
    \toprule
    KL-Shampoo & + LaProp & + BCOS-m \\
    \hline
    $24.95 \pm 0.09$ & $24.94 \pm 0.06$ & $25.15 \pm 0.07$ \\
    \bottomrule
\end{tabular}
\end{table}

In light of these results, it is natural to ask how preconditioning and momentum ought to be rigorously combined. 
Our interpretation of adaptation described in \Cref{sec:description} and elaborated on in \Cref{sec:adapting-descent} suggests that the LaProp variant may be most promising because it cleanly separates adaptation and momentum.
However, how to rigorously handle the two EMAs in Adam and Shampoo remains an open question.
Similar questions can be posed for spectral descent and Muon.

\subsection{The Exponent Paradox}
\label{sec:exponent}

Note that the connection between Shampoo$^{\sfrac{1}{4}}$ without accumulation and spectral descent shown in \Cref{eq:shampoo-spectral} relies on $p=1/4$.
However, $p=\sfrac{1}{2}$ is reported to empirically perform better than $p = \sfrac{1}{4}$ \citep{shi2023distributed,eschenhagen2025purifying}, including in our experiments in \Cref{tab:core-methods}.
When we change the exponent and using the same assumptions, we have
\begin{equation}
\label{eq:shampoo-different-estimator-square-no-beta2}
\begin{split}
    (\mG \mG^\transpose)^{-\frac{1}{2}} \mG (\mG^\transpose \mG)^{-\frac{1}{2}} %
    &= \mU \mSigma^{-1} \mV^\transpose,
\end{split}
\end{equation}
\ie we recover the pseudoinverse instead of the orthogonalized (EMA of the) gradient.
To argue for a pseudoinverse gradient descent ($\pi$GD) based perspective, $\pi$GD should perform better than spectral descent in the same setting.
We ran some preliminary experiments tuning the absolute and relative tolerance of $\pi$GD applied to the EMA of the gradient, but none of the runs came close to matching Muon.
Also, we see no clear reason why the pseudoinverse gradient is a sensible update in this setting.
Instead, investigating whether Shampoo$^{\sfrac{1}{2}}$ is a better approximation of KL-Shampoo (which also uses $p=\sfrac{1}{2}$) than Shampoo$^{\sfrac{1}{4}}$ appears to be a more promising direction (\cf \Cref{prop:instant-kl-shampoo}).

\section{Adapting Sign and Spectral Descent to Stochasticity and the Parameter Trajectory}
\label{sec:adapting-descent}

We can analogously adapt sign and spectral descent to stochasticity and the parameter trajectory by relaxing their enforced constraints and define modified norms on (sequences of) random variables.
It remains unclear, however, if these formulations can be leveraged to generate theoretical insights, \eg to derive meaningful convergence rates.

\subsection{Adapting Sign Descent}
\label{sec:adapting-sign-descent}

\subsubsection{RMSProp's Adapted Optimality Condition}
\label{sec:rmsprop-adapted-optimality}

Let $\E[ \cdot | \vtheta ]$ denote the conditional expectation given deterministic parameters $\vtheta$ and $\vd \in \R^n$ be the update we wish to compute.

Recall that sign descent enforces updates $\vd$ with $d_i^2 = 1$ for all $i = 1, ..., n$. 
(Note that the signed gradient is the closest direction to the gradient with unit $\ell_\infty$ norm).
We can adapt this condition to the stochastic setting and instead enforce $\E[d_i^2] = 1$. 
Observe that the non-time-averaged, idealized (\ie, in expectation) RMSProp algorithm fulfills this condition since we have $\E[(\frac{g_i}{\sqrt{\E[g_i^2]}})^2] = 1$.
The natural relaxation to include adaptation to the parameter trajectory is
\begin{equation}
\label{eq:adapted-sign-optimality}
\boxed{
    \mathrm{EMA}_{t=1}^T (\E[d_i^2 | \vtheta_t]) = (1-\beta_2) \sum_{t=1}^T \beta_2^{T-t} \E[d_i^2 | \vtheta_t] = 1.
}
\end{equation}
Idealized and deterministic RMSProp fulfill this condition as
\begin{equation}
    \mathrm{EMA}_{s=1}^S \left( \E \left[ \left( \frac{g_i}{\sqrt{\mathrm{EMA}_{t=1}^T (\E[g_i^2 | \vtheta_t])}} \right)^2 \Bigg| \vtheta_s \right]  \right) = 1.
\end{equation}

\subsubsection{RMSProp as Descent in an $\ell_\infty(L^2)$ Norm}

More precisely, we can define an $\ell_\infty$ norm over a coordinate-wise $L^2$ norm, i.e.,
\begin{equation}
\label{eq:rmsprop-norm}
\boxed{
\| \vd \|_{\ell_\infty(L^2)} := \| ( \|d_1\|_{L^2}, \dots, \|d_n\|_{L^2} ) \|_{\infty} =\max_{i=1,\dots,n} \sqrt{ \E[ d_i^2 | \vtheta ] } .
}
\end{equation}
When $\vd$ is not a random variable, this norm reduces to the $\ell_\infty$ norm, yielding sign descent in the deterministic setting.
However, since we now consider a random variable, we have to also choose a $L^p$ norm for the corresponding Banach space.
Here, selecting the element-wise $L^2$ norm precisely characterizes \emph{how} RMSProp adapts sign descent to stochasticity.

We choose the update function $\vd$ that minimizes the expected linearized loss at $\vtheta$ subject to the norm constraint:
\begin{equation}
\min_{\vd} \; \E[ \vg^\transpose \vd | \vtheta ] \quad \text{s.t.} \quad \| \vd \|_{\ell_\infty(L^2)} \le 1 .
\end{equation}
The constraint can be written equivalently as
\begin{equation}
\E[ d_i^2 | \vtheta ] \le 1 , \quad i = 1, \dots, n .
\end{equation}
Hence, the problem decouples into $n$ independent one-dimensional problems,
\begin{equation}
\min_{d_i} \; \E[ g_i d_i | \vtheta ] \quad \text{s.t.} \quad \E[ d_i^2 | \vtheta ] \le 1 .
\end{equation}
For each coordinate $i$, by the conditional Cauchy--Schwarz inequality, we have
\begin{equation}
\big| \mathbb{E}[ g_i d_i \mid \theta ] \big| \le \sqrt{ \mathbb{E}[ g_i^2 \mid \theta ] } \; \sqrt{ \mathbb{E}[ d_i^2 \mid \theta ] } ,
\end{equation}
and equality is achieved when $d_i$ is proportional to $g_i$, \ie $d_i = - \alpha_i g_i$ for some scalar $\alpha_i > 0$ (negative sign to minimize the linear term).
Substituting $d_i = -\alpha_i g_i$ into the constraint $\E[ d_i^2 \mid \vtheta ] = 1$ yields
\begin{equation}
\alpha_i^2 \E[ g_i^2 \mid \vtheta ] = 1 
\quad \implies \quad 
\alpha_i = \frac{1}{\sqrt{\E[ g_i^2 \mid \vtheta ]}} 
\quad \implies \quad 
d_i = - \frac{ g_i }{ \sqrt{ \E[ g_i^2 \mid \vtheta ] } }.
\end{equation}
Therefore, the update is
\begin{equation}
\boxed{
\vd = - \frac{ \vg }{ \sqrt{ \E[ \vg^{2} | \vtheta ] } } ,
}
\end{equation}
where the division is element-wise.
Note that in the deterministic case, this will simplify to sign descent.

\subsection{Adapting Spectral Descent}
\label{sec:adapting-spectral-descent}

\subsubsection{Shampoo's Adapted Optimality Condition}
\label{sec:shampoo-adapted-optimality}

In order to derive Shampoo's adapted optimality condition, recall \Cref{def:time-averaged-orthogonality} and \Cref{cor:ideal-kl-shampoo}.
Due to the coupling of $\mA$ and $\mB$ in \Cref{eq:idealized-kl-shampoo}, solving the optimality condition would require storing all prior gradients and computing the EMA at every iteration, in contrast to the simpler case of RMSProp in \Cref{sec:rmsprop-adapted-optimality}.
Instead, we can make a simplifying assumption by using the historical $\mA_{t-1}$ and $\mB_{t-1}$ to compute the current preconditioner recursively, \ie
\begin{equation}
    \hspace{1em} \mA_T = \mathrm{EMA}_{t=1}^{T} \left( \E[ \mG_t \mB_{t-1}^2 \mG_t^\transpose ] \right)^{-\frac{1}{2}}, \hspace{1em} \mB_T = \mathrm{EMA}_{t=1}^{T} \left( \E[ \mG_t^\transpose \mA_{t-1}^2 \mG_t ] \right)^{-\frac{1}{2}}.
\end{equation}
This corresponds to KL-Shampoo.
If $\mL_t$ and $\mR_t$ change slowly such that $\mL_{t-1} \approx \mL_t$ and $\mR_{t-1} \approx \mR_t$ for the terms in the EMA with significant weight, this approximately solves the condition.

\subsubsection{One-sided (KL-)Shampoo without Momentum as Descent in a $S_\infty(L^2)$ Norm}

For a deterministic matrix $\mX$, the spectral norm is defined as $||\mX||_* = \sqrt{ \lambda_\mathrm{max}(\mX \mX^\transpose) } = \sqrt{ \lambda_\mathrm{max}(\mX^\transpose \mX) }$.
However, when $\mX$ is a matrix-valued random variable, $\lambda_\mathrm{max}(\E[\mX \mX^\transpose]) \neq \lambda_\mathrm{max}(\E[\mX^\transpose \mX])$.
This motivates an extension of the spectral norm for random matrices via an inner $L^2$ norm that is analogous to \Cref{eq:rmsprop-norm}:
\begin{equation}
\boxed{
    || \mX ||_{S_\infty(L^2)} = \max\bigl\{ \sqrt{ \lambda_\mathrm{max}(\E[\mX \mX^\transpose]) }, \sqrt{ \lambda_\mathrm{max}(\E[\mX^\transpose \mX]) } \bigr\} .
}
\end{equation}
Trying to derive a steepest descent direction under this norm results in a Sylvester equation, due to the coupled constraint on the row and column space.
Instead, we will focus on constraining only one of the two.

Stochastic steepest descent chooses the update $\mD$ (a deterministic function of the random variable $\mG$) that minimizes the local linearized loss subject to a left-spectral norm constraint:
\begin{equation}
\min_{\mD} \; \E \bigl[ \text{Tr}(\mG^\transpose \mD) \mid \vtheta \bigr] 
\quad \text{s.t.} \quad 
\E[\mD \mD^\transpose \mid \vtheta] \preceq \mI .
\end{equation}
We introduce a symmetric positive semi-definite Lagrange multiplier matrix $\mA \succeq 0$. The Lagrangian is
\begin{equation}
\mathcal{L}(\mD, \mA) = \E \bigl[ \text{Tr}(\mG^\transpose \mD) \bigr] + \text{Tr}\Bigl( \mA \bigl( \E[\mD \mD^\transpose] - \mI \bigr) \Bigr) .
\end{equation}
Using the cyclic property of the trace to group terms under the expectation, we obtain
\begin{equation}
\mathcal{L}(\mD, \mA) = \E \Bigl[ \text{Tr}(\mG^\transpose \mD) + \text{Tr}(\mD^\transpose \mA \mD) \Bigr] - \text{Tr}(\mA),
\end{equation}
and by computing the gradient and setting it to zero, we have
\begin{equation}
\nabla_{\mD} \mathcal{L} = \mG + 2 \mA \mD = 0 \quad \implies \quad \mD = -\frac{1}{2} \mA^{-1} \mG .
\end{equation}
Absorbing the scalar factor into the multiplier (redefining $\mA \leftarrow 2\mA$), we have the structural form $\mD = -\mA^{-1} \mG$. To determine $\mA$, we substitute this into the constraint $\E[\mD \mD^\transpose] = \mI$:
\begin{equation}
\E \bigl[ (\mA^{-1} \mG) (\mA^{-1} \mG)^\transpose \bigr] = \mA^{-1} \E[\mG \mG^\transpose] \mA^{-1} = \mI .
\end{equation}
Therefore, the steepest descent update is
\begin{equation}
\boxed{
\mD = - \E[ \mG \mG^\transpose \mid \vtheta ]^{-\frac{1}{2}} \mG,
}
\end{equation}
assuming $\E[ \mG \mG^\transpose \mid \vtheta ]$ is invertible.
We recover one-sided (KL-)Shampoo without time-averaging and an expectation instead of a single sample estimate.

\end{document}